\documentclass[twoside,11pt]{article}
\pdfoutput=1

\usepackage{jmlr2e}
\usepackage{amsmath}
\usepackage{bm}
\usepackage{color}
\usepackage{tikz}
\usepackage{hyperref}
\usepackage{natbib}
\usepackage[linesnumbered]{algorithm2e}
\usetikzlibrary{arrows.meta,calc,decorations.markings,math,arrows.meta}
\usepackage{pgfplots}

\DeclareMathOperator*{\argmax}{arg\,max} 
\newcommand\numberthis{\addtocounter{equation}{1}\tag{\theequation}}
\def\ci{\perp\!\!\!\perp}

\def\circlinecirc{\hspace{1mm}{\circ\! {-} \! \circ}\hspace{1mm}}
\def\linecirc{\hspace{1mm}{-\!\circ}\hspace{1mm}}

\pgfmathdeclarefunction{gauss}{2}{%
  \pgfmathparse{1/(#2*sqrt(2*pi))*exp(-((x-#1)^2)/(2*#2^2))}%
}

\jmlrheading{1}{2016}{}{}{}{Eric V. Strobl and Shyam Visweswaran}

\ShortHeadings{PC with p-values}{Strobl, Spirtes and Visweswaran}
\firstpageno{1}

\begin{document}

\title{Estimating and Controlling the False Discovery Rate of the PC Algorithm Using Edge-Specific P-Values}

\author{\name Eric V. Strobl \email evs17@pitt.edu \\
       \addr Department of Biomedical Informatics\\
       University of Pittsburgh\\
       Pittsburgh, PA 15206, USA
       \AND
\name Peter L. Spirtes \email ps7z@andrew.cmu.edu \\
       \addr Department of Philosophy\\
       Carnegie Mellon University\\
       Pittsburgh, PA 15213, USA
       \AND
	   \name Shyam Visweswaran \email shv3@pitt.edu \\
       \addr Department of Biomedical Informatics\\
       University of Pittsburgh\\
       Pittsburgh, PA 15206, USA}

\editor{TBA}

\maketitle

\begin{abstract}
The PC algorithm allows investigators to estimate a complete partially directed acyclic graph (CPDAG) from a finite dataset, but few groups have investigated strategies for estimating and controlling the false discovery rate (FDR) of the edges in the CPDAG. In this paper, we introduce PC with p-values (PC-p), a fast algorithm which robustly computes edge-specific p-values and then estimates and controls the FDR across the edges. PC-p specifically uses the p-values returned by many conditional independence (CI) tests to upper bound the p-values of more complex edge-specific hypothesis tests. The algorithm then estimates and controls the FDR using the bounded p-values and the Benjamini-Yekutieli FDR procedure. Modifications to the original PC algorithm also help PC-p accurately compute the upper bounds despite non-zero Type II error rates. Experiments show that PC-p yields more accurate FDR estimation and control across the edges in a variety of CPDAGs compared to alternative methods\footnote{MATLAB implementation: https://github.com/ericstrobl/PCp/}.
\end{abstract}


\section{Introduction} \label{sec_1}

Discovering causal relationships is often much more important than discovering associational relationships in the sciences. As a result, the research community has been conducting extensive investigations into causal inference with the hope of developing practically useful algorithms to speed-up the scientific process. This research has resulted in a wide range of high performing algorithms over the years such as PC \citep{Spirtes00}, FCI \citep{Spirtes95, Spirtes00}, and CCD \citep{Richardson96}

The PC algorithm is currently one of the most popular methods for inferring causation from observational data. Given an observational dataset, the algorithm outputs a complete partially directed acyclic graph (CPDAG) which has helped some investigators elucidate important causal relationships in several domains. For example, the PC algorithm has been used to discover new causal relationships between genes and brain regions in biology \citep{Wu06,Li08,Joshi10,Sun12,Harris13,Iyer13,Le13,Teramoto14,Ha15}. The algorithm has also been used to discover causal relations between corporate structures and strategies in economics \citep{Chong09} as well as academic and musical achievements in psychology \citep{Mullensiefen15}.

The increased use of PC in recent years has nonetheless led to growing concern about algorithm's confidence level in each edge of the CPDAG. For example, PC may have more confidence in the edge $A-B$ but have less confidence in the edge $B \rightarrow C$ in the subgraph $A-B \rightarrow C$ of the CPDAG. Currently, PC alone does not output any edge-specific measure of confidence, even though scientists often must report measures of confidence such as p-values or confidence intervals in their scientific articles. This incongruency has resulted in the relatively slow adoption or even avoidance of the PC algorithm in the sciences, despite the algorithm's impressive capabilities in causal inference. Clearly then rectifying the problem by developing an edge-specific measure of confidence will increase the adoption of PC as well as hopefully ease the transition of ever-more complex causal inference algorithms into the scientific community.

We can of course consider multiple different ways of representing the confidence level in each edge of the CPDAG. However, we choose to pay special attention to the p-value, since it is by far the most popular notion of confidence in the sciences. Indeed, nearly all scientists report p-values in modern scientific reports because they rely on p-values to help justify their hypotheses. We therefore would ideally like to assign a p-value to each edge in the CPDAG as in Figure \hyperref[fig_pvalues]{1a} in order to best integrate the algorithm within a well-known framework. In this paper, we propose such a ``causal p-value'' in detail.

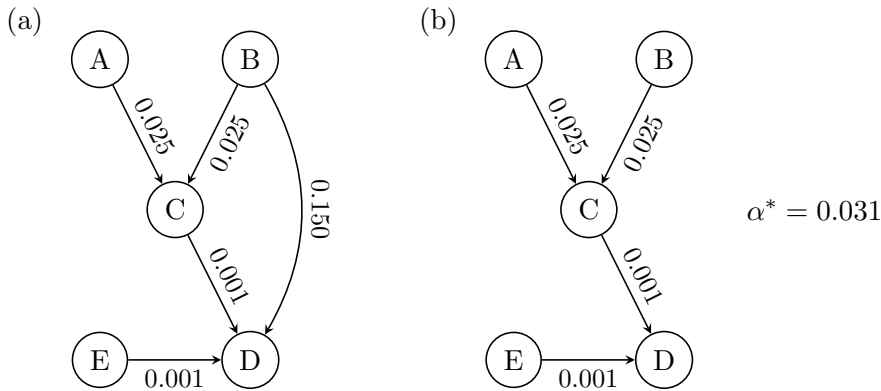
\begin{figure}[b!]
\centering 
\begin{tikzpicture}[
            > = stealth, % arrow head style
            auto,
            semithick % line style
        ]
        \node[shape=circle] (F) at (-1,0.5) {(a)};
    \node[shape=circle,draw=black] (A) at (0,0) {A};
    \node[shape=circle,draw=black] (B) at (2,0) {B};
    \node[shape=circle,draw=black] (C) at (1, -2) {C};
    \node[shape=circle,draw=black] (D) at (2, -4) {D};
    \node[shape=circle,draw=black] (E) at (0,-4) {E};

    \path [->] (A) edge node[sloped,above] {\small{0.025}}(C);
    \path [->] (B) edge node[sloped,below] {\small{0.025}}(C);
    \path [->] (C) edge node[sloped,above] {\small{0.001}}(D);
    \path [->] (E) edge node[sloped,below] {\small{0.001}}(D);
    \path [->, bend left] (B) edge node[sloped,above] {\small{0.150}}(D);
    
        \node[shape=circle] (F) at (4.5,0.5) {(b)};
        \node[shape=circle,draw=black] (A) at (5.5,0) {A};
    \node[shape=circle,draw=black] (B) at (7.5,0) {B};
    \node[shape=circle,draw=black] (C) at (6.5, -2) {C};
    \node[shape=circle,draw=black] (D) at (7.5, -4) {D};
    \node[shape=circle,draw=black] (E) at (5.5,-4) {E};
        \node[shape=circle] (F) at (9.5,-2) {$\alpha^* = 0.031$};

    \path [->] (A) edge node[sloped,above] {\small{0.025}}(C);
    \path [->] (B) edge node[sloped,below] {\small{0.025}}(C);
    \path [->] (C) edge node[sloped,above] {\small{0.001}}(D);
    \path [->] (E) edge node[sloped,below] {\small{0.001}}(D);
    
\end{tikzpicture}
\caption{We seek to associate edge-specific p-values to the output of the PC algorithm such as in (a). The PC algorithm currently does not associate such p-values with its output. We would also like to control the FDR of the edges. In (b), we set the FDR to $0.1$ and obtained a $\alpha^*$ cutoff of $0.031$ for the output in (a), so we eliminated the edge between $B$ and $D$ because its p-value exceeds $\alpha^*$.}
\label{fig_pvalues}
\end{figure}

We however also believe that assigning p-values to each edge is not enough to ease the transition of PC into mainstream science, since the CPDAG actually contains many edges and therefore also represents a complicated multiple hypothesis testing problem. Fortunately, the problem of multiple hypothesis testing has a long history, as scientists have often required the results of multiple hypothesis tests in order to answer complex scientific questions. Currently, a standard approach to tackling the multiple hypothesis testing problem involves controlling the proportion of false positives among the rejected null hypotheses, or the \textit{false discovery rate (FDR)}, by using an \textit{FDR controlling procedure} that takes a desired FDR level $q$ and a set of p-values as input. The procedure then outputs a corresponding significance level $\alpha^*$ for the set of p-values. An investigator subsequently rejects the null hypotheses for those tests with p-values that fall below $\alpha^*$ in order to ensure that the expected FDR does not exceed $q$. For example, consider the set of p-values $\{0.02, 0.01, 0.03\}$ and suppose that the FDR controlling procedure with $q=0.1$ outputs $\alpha^*=0.019$. Then, rejecting the null hypotheses of the first and third hypothesis tests guarantees that the expected FDR does not exceed 10\%. Several other FDR controlling strategies also exist, but the ease of use, speed and accuracy of the above method have made it the most widely adopted strategy in the last two decades. 

We would therefore like to control the FDR in the edges of the CPDAG using an FDR controlling procedure like in Figure \hyperref[fig_pvalues]{1b}. As a first idea, one may wonder whether an investigator can control the FDR in the CPDAG by simply feeding in all of the CI test p-values computed by PC into an FDR controlling procedure. Unfortunately, this approach fails for at least two reasons. First, it is unclear how to use the p-values which exceed the $\alpha^*$ cut-off to reject edges in the CPDAG, since one would first need to elucidate the correspondence between the p-values and edges. Second, even if one could solve this problem, the strategy may only loosely bound the expected FDR. An accurate FDR controlling procedure should instead take into account the specific computations executed by PC in order to identify a sharp bound. The p-value based approach therefore necessitates a more fine-grained strategy which has thus far remained undiscovered.

Several groups have nonetheless attempted to control the FDR in the CPDAG by avoiding the complicated nature of the above problem with a different, data re-sampling approach. For example, Friedman and colleagues proposed to estimate the FDR by using the parametric bootstrap \citep{Friedman99}. This procedure involves first learning a causal graph with the PC algorithm. The procedure then generates data from the causal graph and re-applies the PC algorithm multiple times on each generated dataset to estimate the FDR using the learnt causal graphs. An investigator can subsequently control the FDR by repeating the above process with different $\alpha$ values until he or she reaches the desired FDR level $q$. However, notice that the method requires multiples calls to PC and can therefore require too much time with high dimensional data. The procedure also requires parametric knowledge about the underlying distribution which limits the applicability of the method to simple cases. Fortuitously, two groups later proposed a permutation-based method which drops the parametric assumption \citep{Listgarten07,Armen14}. The permutation method nevertheless also requires multiple calls to an algorithm and in fact only applies to the parts of PC which can be decomposed into independent searches for the parents of each vertex; this has thus far limited the applicability of the method to adjacency discovery with local to global discovery algorithms (e.g., MMHC) and incomplete edge orientation. We conclude that both the bootstrap and permutation approaches to FDR estimation and control are either incomplete or too time consuming. 

Another class of methods fortunately attempts to control the FDR without resampling procedures by instead using a standard FDR controlling procedure with bounded p-values. For instance, one method proposed in \citep{Tsamardinos08} and then refined in \citep{Armen11,Armen14}  assigns a p-value to each adjacency by taking the maximum over all of the significant p-values from the associated CI tests executed by PC. The method then controls the FDR in the estimated adjacencies by applying an FDR controlling procedure, such as the one proposed by Benjamini and Yekutieli (BY) \citep{Benjamini01}, on the edge-specific p-values. Under faithfulness and a zero Type II error rate, the method controls the FDR across the estimated adjacencies, or the estimated \textit{skeleton} \citep{Armen14}. This two stage method also performs comparably with the one stage method proposed in \citep{Li08, Li09}, which controls the FDR during, as opposed to after, the execution of the skeleton discovery phase of the PC algorithm. Of course, the Type II error rate never reaches zero in practice but researchers have also investigated a strategy for reducing the realized Type II error rate by introducing a heuristic reliability criterion for CI tests when dealing with discrete data \citep{Armen14}. Experiments have shown that these methods finish in a relatively short amount of time and perform well in practice. However, the methods are also incomplete because they only apply to the skeleton discovery phase of PC.

In this report, we build on the previous outstanding work for deriving p-values for adjacencies by contributing a sound, complete and fast algorithm called PC with p-values (PC-p) which appropriately combines the p-values of PC's CI tests and then uses the BY FDR controlling procedure to accurately control the FDR in a CPDAG. The method relies on two upper bounds of the p-value that relate to logical conjunctions and disjunctions as described in Section \ref{sec_3}. These upper bounds allow us to formulate several hypothesis tests for recovering the skeleton, discovering unshielded v-structures, and orienting additional edges as presented in Section \ref{sec_4}. Accurately estimating the p-values of the hypothesis tests nonetheless requires a modified version of PC called PC-p which we propose in Section \ref{sec_5}. Finally, we provide experimental results in Section \ref{sec_6} which show that PC-p's p-value estimates yield accurate estimates of the FDR with the BY procedure and improve upon alternative methods.

\section{Preliminaries} \label{sec_2}
\subsection{Causal graphs}

A \textit{causal graph} consists of vertices representing variables and edges representing causal relationships between any two variables. In this paper, we will use the terms ``vertices" and ``variables" interchangeably. \textit{Directed graphs} are graphs where two distinct vertices can be connected by edges ``$\to$" and ``$\gets$." We only consider \textit{simple graphs} in this paper, or graphs with no edges originating from and connecting to the same vertex. \textit{Directed acyclic graphs} (DAGs) are directed graphs without directed cycles. We say that $X$ and $Y$ are \textit{adjacent} if they are connected by an edge independent of the edge's direction. A \textit{path} $p$ from $X$ to $Y$ is a set of consecutive edges (also independent of their direction) from $X$ to $Y$ such that no vertex is visited more than once. Given a path between two vertices $X$ and $Y$ with a middle vertex $Z$, the path is a \textit{chain} if $X \to Y \to Z$, a \textit{fork} if $X \gets Y \to Z$, and a \textit{v-structure} if $X \to Y \gets Z$. We refer to $Y$ as a \textit{collider}, if it is the middle vertex in a v-structure. A v-structure is called an \textit{unshielded v-structure} if $X \to Y \gets Z$, but $X$ and $Z$ are non-adjacent. A \textit{directed path} from $X$ to $Y$ is a set of consecutive edges with direction. We say that $X$ is an \textit{ancestor} of $Y$ (and $Y$ is a \textit{descendant} of $X$), if there exists a directed path from $X$ to $Y$. 

If $\mathbb{G}$ is a directed graph in which $\bm{X}$, $\bm{Y}$ and $\bm{Z}$ are disjoint sets of vertices, then $\bm{X}$ and $\bm{Y}$ are \textit{d-connected} by $\bm{Z}$ in $\mathbb{G}$ if and only if there exists an undirected path $p$ between some vertex in $\bm{X}$ and some vertex in $\bm{Y}$ such that, for every collider $C$ on $p$, either $C$ or a descendant of $C$ is in $\bm{Z}$, and no non-collider on $p$ is in $\bm{Z}$. On the other hand, $\bm{X}$ and $\bm{Y}$ are \textit{d-separated} by $\bm{Z}$ in $\mathbb{G}$ if and only if they are not d-connected by $\bm{Z}$ in $\mathbb{G}$.  Next, the joint probability distribution $\mathbb{P}$ over variables $\bm{X}$ satisfies the \textit{global directed Markov property} for a directed graph $\mathbb{G}$ if and only if, for any three disjoint subsets of variables $\bm{A}$, $\bm{B}$ and $\bm{C}$ from $\bm{X}$, if $\bm{A}$ and $\bm{B}$ are d-separated given $\bm{C}$ in $\mathbb{G}$, then $\bm{A}$ and $\bm{B}$ are conditionally independent given $\bm{C}$ in $\mathbb{P}$. We refer to the converse of the global directed Markov property as \textit{d-separation faithfulness}; that is, if $\bm{A}$ and $\bm{B}$ are conditionally independent given $\bm{C}$ in $\mathbb{P}$, then $\bm{A}$ and $\bm{B}$ are d-separated given $\bm{C}$ in $\mathbb{G}$. 

A \textit{Markov equivalence class} of DAGs refers to a set of DAGs which entail the same conditional independencies. A \textit{complete partially directed acyclic graph} (CPDAG) is a partially directed acyclic graph with the following properties: (1) each directed edge exists in every DAG in the Markov equivalence class, and (2) there exists a DAG with $X \to Y$ and a DAG with $X \gets Y$ in the Markov equivalence class for every undirected edge $X-Y$. A CPDAG $\mathbb{G}^C$ represents a DAG $\mathbb{G}$, if $\mathbb{G}$ belongs to the Markov equivalence class described by $\mathbb{G}^C$. We will occasionally use the meta-symbol ``$\circ$'' at the endpoint(s) of an edge to denote the presence or absence of an arrowhead. For example, the edge ``$\linecirc$'' may denote either ``$-$'' or ``$\to$''.  

\subsection{The PC Algorithm} \label{sec_2_2}
The PC algorithm is comprised of three stages. We have summarized these stages as pseudocode in Algorithms \ref{pc_1}, \ref{pc_2} and \ref{pc_3} in Section \ref{appendix_1} of the Appendix. The first stage estimates the adjacencies of $\mathbb{G}$, or the skeleton of $\mathbb{G}$. Starting with a fully connected skeleton, the algorithm attempts to eliminate the adjacency between any two variables, say $A$ and $B$, by testing if $A$ and $B$ are conditionally independent given some subset of the neighbors of $A$ or the neighbors of $B$. The search is performed progressively, whereby the algorithm increases the size of the conditioning set starting from zero using a step size of 1. The edge between $A$ and $B$ is removed, if $A$ and $B$ are rendered conditionally independent given some subset of the neighbors of $A$ or the neighbors of $B$.

The PC algorithm orients unshielded colliders in its second stage. Specifically, PC finds triples $A$, $B$, $C$ such that $A-B-C$, but $A$ and $C$ are non-adjacent. The algorithm then determines whether $B$ is contained in the set which rendered $A$ and $C$ conditionally independent in the first stage of PC. If not, $A-B-C$ is replaced with $A \to B \gets C$. 

The third and final stage of PC involves the repetitive application of three rules to orient as many of the remaining undirected edges as possible. The three rules include:
\begin{equation} \label{rules1}
\begin{aligned}
1. \hspace{3mm} & \text{If } A-B \text{, } C \to A \text{ and } C \text{ and } B \text{ are non-adjacent, then replace } A-B \text{ with } \\ & A \to B. \\
2. \hspace{3mm} & \text{If } A-B \text{ and }  A \to C \to B \text{, then replace } A-B \text{ with } A \to B. \\
3. \hspace{3mm} & \text{If } A-B \text{, } A-C \to B \text{, } A-D \to B \text{, and } C \text{ and } D \text{ are non-adjacent, } \\ \hspace{3mm} &\text{then replace } A-B \text{ with } A \to B.
\end{aligned}
\end{equation}

Overall, the PC algorithm has been shown to be complete in the sense that it finds and then orients edges up to $\mathbb{G}^C$, a CPDAG that represents $\mathbb{G}$ \citep{Meek95}.

\subsection{Hypothesis Testing}

A \textit{hypothesis test} is a method of statistical inference usually composed of one \textit{null} ($H_0$) and one \textit{alternative} ($H_1$) \textit{hypothesis} which are mutually exclusive; that is, if one occurs, then the other cannot occur. The null hypothesis refers to the default position which asserts that whatever one is trying to statistically infer actually did not happen. Note that the null and alternative do not necessarily need to be logical complements of each other. For example, one may be interested in determining whether the parameter $\mu$ is greater than zero. In this case, the null can be defined as $\mu=0$ while the alternative can be defined as $\mu>0$ instead of $\mu \not = 0$. 

A \textit{Type I error} is the incorrect rejection of a true null hypothesis, or a false positive.  On the other hand, a \textit{Type II error} is the failure to reject a false null hypothesis, or a false negative. The \textit{p-value} ($p$) is the probability of the Type I error, or the Type I error rate. More specifically, the p-value is the probability of obtaining a result equal to or more extreme than the observed value under the assumption of the null hypothesis. The null hypothesis is thus rejected when the p-value is at or below a predefined \textit{$\alpha$ threshold} (typically the $\alpha$ threshold is set to 0.05), because a low p-value demonstrates the improbability of the null hypothesis.

\subsection{False Discovery Rate}

\textit{Multiple comparisons} or \textit{multiple hypothesis testing} refers to the process of considering more than one statistical inference simultaneously. Failure to compensate for multiple comparisons can result in erroneous inferences. For example, if an investigator performs one hypothesis test with an $\alpha$ threshold of 0.05, then he or she has only a $5\%$ chance of making a Type I error. However, if the investigator performs 100 independent tests with the same $\alpha$ threshold, then he or she has a $1-(1-0.05)^{100}=99.4\%$ chance of making a Type I error on at least one test.  

In multiple hypothesis testing, the \textit{false discovery rate} (FDR) at threshold $\alpha$ is the expected proportion of false positives among the rejected null hypotheses. Specifically, we define the FDR at $\alpha$ as follows:
\begin{equation} \nonumber
FDR(\alpha) \triangleq \mathbb{E}\left[\frac{V}{\max\{R, 1\}}\right],
\end{equation}
where $V$ is the number of false positives, $R$ is the total number of null hypotheses rejected, and $\max\{R, 1\}$ ensures that $FDR(\alpha)$ is well-defined when $R=0$. We define the \textit{realized FDR} at $\alpha$ as $V/\max\{R, 1\}$.

\textit{FDR estimation}, or \textit{conservative point estimation of the FDR}, refers to the process of estimating $FDR(\alpha)$ in a conservative manner such that:
\begin{equation} \nonumber
\mathbb{E}[\widehat{FDR}(\alpha)] \geq FDR(\alpha),
\end{equation}
where $\widehat{FDR}(\alpha)$ represents an estimate of $FDR(\alpha)$. We denote $\mathbb{E}[\widehat{FDR}(\alpha)]-FDR(\alpha)$ as the \textit{estimation bias}. Note that there are several ways of obtaining $\widehat{FDR}(\alpha)$. In 2001, Benjamini and Yekutieli proposed the following \textit{FDR estimator} for $m$ hypothesis tests: 
\begin{equation} \label{FDR_BY}
\widehat{FDR}_{BY}(\alpha) \triangleq \frac{m\alpha\sum_{i=1}^{m}\frac{1}{i}}{\max\{R, 1\}}.
\end{equation}
FDR estimators such as $\widehat{FDR}_{BY}$ can be used to define \textit{FDR controlling procedures}. These procedures determine the optimal threshold $\alpha^*$ which achieves \textit{strong control}\footnote{\textit{Strong control} of the FDR refers the process of controlling the FDR under any configuration of true and false null hypotheses; on the other hand, \textit{weak control} refers to the process of controlling the FDR when all of the null hypotheses are true. Strong control is therefore preferable to weak control.} of the FDR in the following sense:
\begin{equation} \label{alpha_star}
\alpha^* \triangleq \argmax_{\alpha} \{\widehat{FDR}(\alpha) \leq q\}
\end{equation}
The FDR controlling procedure based on $\widehat{FDR}$ involves the rejection of all null hypotheses with p-values below the $\alpha^*$ threshold. We refer to the quantity $FDR(\alpha^* )-q$ as the \textit{control bias}. Benjamini and Yekutieli proved that the estimate $\widehat{FDR}_{BY}$ in particular achieves strong control of the FDR with any form of dependence among the p-values of $m$ hypothesis tests. 

\section{Upper Bounds on the P-Value} \label{sec_3}
We present two upper bounds of the Type I error rate of hypothesis tests which can be constructed using a set of simpler hypothesis tests. These upper bounds will serve as useful tools in Section \ref{sec_4} for bounding the Type I error rate of the hypothesis tests which will be used to infer the presence or absence of edges in a CPDAG. 
\subsection{Union Bound}
Consider the following hypothesis test for two random variables given a conditioning set:
\begin{align*}
H_0: &\text{ Conditionally independent}, \\
H_1: &\text{ Conditionally dependent}.
\end{align*}
Trivially, we can rephrase the null and alternative in terms of a conditional independence (CI) oracle:
\begin{align*}
H_0: &\text{ The CI oracle outputs independent}, \\
H_1: &\text{ The CI oracle outputs dependent}.
\end{align*}
Now suppose we want to query $m$ CI oracles about $m$ CI relations. We can then consider the following null and alternative:
\begin{equation} \label{union_CI}
\begin{aligned}
H_0: &\text{ All CI oracles output independent}, \\
H_1: &\text{ At least one CI oracle outputs dependent}.
\end{aligned}
\end{equation}
From here on, we write Pr(CI test $i$ outputs dependent $|$ CI oracle $i$ outputs independent) to denote Pr($|\widehat{s}_i| \geq s_i^\alpha \big| s_i=0$), where $s_i$ refers to a parameter of some standardized distribution used by CI test $i$, $\widehat{s}_i$ a random variable and the test statistic estimating $s_i$, and $s_i^\alpha$ a value of $s_i$ determined by an $\alpha$ level. We provide an example in Figure \ref{fig_stand_norm}, where $s_i$ may correspond to Fisher's $z$-statistic in the case of Fisher's $z$-test for the mean parameter $s_i=0$ of the standard normal distribution.

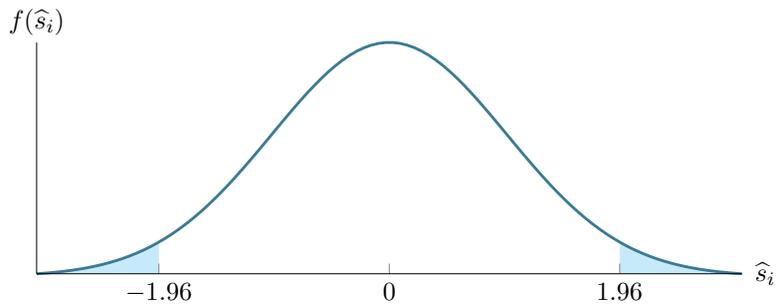
\begin{figure}
\centering
\begin{tikzpicture}[scale=0.9]
\begin{axis}[
  no markers, domain=-3:3, samples=100,
  axis lines*=left, xlabel=$\widehat{s}_i$, ylabel=$f(\widehat{s}_i)$,
  every axis y label/.style={at=(current axis.left of origin),above=34mm},
  every axis x label/.style={at=(current axis.right of origin),right=11mm},
  height=5cm, width=12cm,
  xtick={0, -1.96, 1.96}, ytick=\empty,
  enlargelimits=false, clip=false, axis on top
  ]
  \addplot [fill=cyan!20, draw=none, domain=-3:-1.96] {gauss(0,1)} \closedcycle;
  \addplot [fill=cyan!20, draw=none, domain=1.96:3] {gauss(0,1)} \closedcycle;
  \addplot [very thick,cyan!50!black] {gauss(0,1)};

%\draw [yshift=-0.6cm, latex-latex](axis cs:3,0) -- node [fill=white] {Reject} (axis cs:1.96,0);
\end{axis}
\end{tikzpicture}
\caption{In the above standard normal case, we have Pr$(|\widehat{s}_i| \geq s_i^\alpha \big| s_i=0)=0.05$, where $s_i^\alpha = 1.96$. We reject the null hypothesis when $|\widehat{s}_i|$ falls in the blue colored regions at the tails.}
\label{fig_stand_norm}
\end{figure}

We now bound the Type I error rate of the hypothesis test \eqref{union_CI} by using the new notation and the union bound:
\begin{align*} \label{union_bound_der}
&\hspace{5mm}\text{Pr(Type I error)}\\
&=\text{Pr(at least one CI test outputs dependent}|H_0) \\
&=\text{Pr}\left(\bigvee_{i=1}^m\text{CI test $i$ outputs dependent}|H_0 \right) \\
&\leq \sum_{i=1}^m \text{Pr(CI test $i$ outputs dependent}|H_0 ) \\
&= \sum_{i=1}^m \text{Pr(CI test $i$ outputs dependent}|\text{CI oracle $i$ outputs independent}) \\
&= \sum_{i=1}^m p_i, \numberthis
\end{align*}
where $p_i$ denotes the Type I error rate of CI test $i$. Thus, if the r.h.s. of \eqref{union_bound_der} is less than the $\alpha$ threshold, then we can conclude that the Type I error rate of \eqref{union_CI} is also below the threshold. In other words, \eqref{union_bound_der} is a conservative p-value.

Note that the third equality in the derivation of \eqref{union_bound_der} uses the simplifying assumption that the probability of the output of CI test $i$ only depends on the output of CI oracle $i$ when given the outputs of all CI oracles. Several papers have used this assumption implicitly in their proofs \citep{Tsamardinos08,Li09}, and we will also use it throughout this paper. We can justify the assumption based on three facts. First, most CI test statistics $s_i$ have a limiting distribution which only depends on $s_i=0$ under the null. For example, Fisher's $z$-statistic has a limiting standard normal distribution with mean parameter $z_i=0$ and constant variance. Moreover, the $G$-statistic for the $G$-test has a limiting $\chi^2$-distribution with non-centrality parameter $g_i=0$ and degrees of freedom determined by the number of cells in the contingency table. Second, existing methods which utilize bounds based on the assumption have strong empirical performance; loose-enough bounds therefore appear to accommodate the assumption well in most finite sample cases. Third, recall that simplifying assumptions are not new in the causality literature; indeed, many authors have made simplifying assumptions regarding parameter independence for Bayesian methods which similarly increase computational efficiency and achieve strong empirical performance (e.g., \citep{Cooper99}).

We can now also generalize the bound in \eqref{union_bound_der} to any hypothesis test consisting of a series of logical disjunctions in the alternative and a series of logical conjunctions in the null. Namely:
\begin{equation} \label{gen_union}
\begin{aligned}
H_0: &\bigwedge_{i=1}^m \text{oracle $i$ outputs $\neg P_i$}, \\
H_1: &\bigvee_{i=1}^m \text{oracle $i$ outputs $P_i$},
\end{aligned}
\end{equation}
where $P_i$ denotes an arbitrary output of oracle $i$. We now have:
\begin{align*} 
&\hspace{5mm} \text{Pr(Type I error)}=\text{Pr(}\bigvee_{i=1}^m \text{test $i$ outputs $P_i$ }〗|H_0 ) \\
&\leq \sum_{i=1}^m \text{Pr(test $i$ outputs $P_i$} |H_0 ) \\
&= \sum_{i=1}^m \text{Pr(test $i$ outputs $P_i$} | \text{oracle $i$ outputs $\neg P_i$} ) \\
& = \sum_{i=1}^m h_i ,
\end{align*}
where $h_i$ is the Type I error rate of test $i$, and the second equality uses the assumption that the probability of the output of test $i$ only depends on the output of oracle $i$ when given all oracles. We will use this generalization in Section \ref{sec_4}.

\subsection{Intersection Bound}
Suppose we want to perform a hypothesis test with the following null and alternative which are different than the null and alternative in \eqref{union_CI}:
\begin{equation} \label{max_bound_CI}
\begin{aligned}
H_0: &\text{ At least one CI oracle outputs independent}, \\
H_1: &\text{ All CI oracles output dependent}.
\end{aligned}
\end{equation}
Now assume we know that the $i^{\text{th}}$ CI oracle outputs independent. We can then bound the Type I error rate of \eqref{max_bound_CI} as follows with $m$ queries to the CI oracle:
\begin{equation} \label{max_bound_der1}
\begin{aligned}
&\hspace{5mm}\text{Pr(Type I error)}=\text{Pr(all $m$ CI tests output dependent}|H_0 ) \\
&\leq \text{Pr(CI test $i$ outputs dependent } | H_0) \\
&= \text{Pr(CI test $i$ outputs dependent } | \text{ CI oracle $i$ outputs independent} \\ & \hspace{5mm}\wedge \text{other CI oracles \textit{may} output independent}) \\
&= \text{Pr(CI test $i$ outputs dependent } | \text{ CI oracle $i$ outputs independent}) \\
&=p_i,
\end{aligned}
\end{equation}
where the second equality again holds under the assumption that the probability of the output of CI test $i$ only depends on the output of CI oracle $i$ when given the outputs of all CI oracles. We can therefore bound the Type I error rate of \eqref{max_bound_CI} using the p-value of a single CI test for which the CI oracle outputs independent. Nevertheless, in practice, we often do not know for which query the oracle outputs independent in the null. We do however know that at least one unknown CI oracle $i$ outputs independent, so we can bound the Type I error rate of \eqref{max_bound_CI} using the maximum over all of the $m$ CI test p-values:
%\begin{equation} \label{max_bound_der2}
\begin{align*} \label{max_bound_der2}
&\hspace{5mm}\text{Pr(Type I error)}=\text{Pr(all $m$ CI tests output dependent}|H_0 ) \\ \nonumber
&\leq\text{Pr(CI test $i$ outputs dependent } | \text{ CI oracle $i$ outputs independent} ) \\ \nonumber
&=p_i \leq \max_{j=1,\dots,m} p_j. \numberthis
\end{align*}
%\end{equation}

Note that we can generalize the above bound to any hypothesis test consisting of a series of logical conjunctions in the alternative and a series of logical disjunctions in the null. Namely:
\begin{equation} \label{max_bound_gen}
\begin{aligned}
H_0: &\bigvee_{i=1}^m \text{oracle } i \text{ outputs } \neg P_i, \\
H_1: &\bigwedge_{i=1}^m \text{oracle } i \text{ outputs } P_i.
\end{aligned}
\end{equation}
We therefore have:	
\begin{align*} \label{max_bound_gen2}
\text{Pr(Type I error)}&=\text{Pr(}\bigwedge_{i=1}^m \text{test $i$ outputs $P_i$} |H_0 ) \\ \nonumber
&\leq\text{Pr(test $i$ outputs $P_i$}  | H_0) \\\nonumber
&=\text{Pr(test $i$ outputs $P_i$}  | \text{oracle $i$ outputs $\neg P_i$ )} \\\nonumber
&= h_i \leq \max_{j=1,\dots,m} h_j,\numberthis
\end{align*}
where the second equality again uses the assumption that the probability of the output of test $i$ only depends on the output of oracle $i$ when given all oracles.

\section{Edge-Specific Hypothesis Tests} \label{sec_4}
We now show how to apply the two upper bounds of the Type I error rate to derive p-value estimates for both the undirected and directed edges in the CPDAG as estimated by PC. Bounding the p-value for each edge therefore amounts to adding up and/or maximizing over the p-values returned from multiple CI tests.

Note that we will sometimes invoke a zero Type II error rate assumption in this section. This assumption is necessary to correctly upper bound the p-values of the edge-specific hypothesis tests of the CPDAG according to the CI tests executed by the PC algorithm. In fact, we can always correctly bound the p-values, if we perform all of the possible CI tests between the considered variables; however, this approach is impractical, since it ignores the efficiencies of the PC algorithm. A more interesting strategy involves designing the edge-specific hypothesis tests so that the p-value bounds are robust to Type II errors as well as redesigning the PC algorithm to catch many Type II errors. We will discuss these approaches in detail in Sections \ref{sec_4_4} and \ref{sec_5}, so we encourage readers to accept the zero Type II error rate assumption for now.

\subsection{Skeleton Discovery} \label{sec_4_1}
We first consider the skeleton discovery phase of the PC algorithm. We wish to test whether each edge is absent in the true skeleton starting from a completely connected undirected graph. This problem has already been investigated in \citep{Li08,Tsamardinos08,Li09,Armen11,Armen14}, but we review it here for completeness. We construct a hypothesis test with the following null and alternative:
\begin{equation} \label{undirected_HT}
\begin{aligned}
H_0: & \text{ } A-B \text{ is absent}, \\
H_1: & \text{ } A-B \text{ is present}.
\end{aligned}
\end{equation}
Now consider the following proposition, where $\bm{Pa}(A)$ denotes the true parents of $A$:
\begin{proposition} \label{thm_adj}
\citep{Spirtes00} Consider a DAG $\mathbb{G}$ which satisfies the global directed Markov property. Moreover, assume that the probability distribution is d-separation faithful. Then, there is an edge between two vertices $A$ and $B$ if and only if $A$ and $B$ are conditionally dependent given any subset of $\bm{Pa}(A) \setminus B$ and any subset of $\bm{Pa}(B) \setminus A$.
\end{proposition}
We thus consider the following two scenarios for the undirected edge $A-B$:
\begin{equation} \label{rules2}
\begin{aligned}
1. \hspace{3mm} & \text{If } A \text{ and } B \text{ are conditionally independent given some subset of } \bm{Pa}(A)\setminus B \\ 
& \text{or some subset of } \bm{Pa}(B)\setminus A \text{, then } A-B \text{ is absent.} \\
2. \hspace{3mm} & \text{If } A \text{ and } B \text{ are conditionally dependent given any subset of } \bm{Pa}(A)\setminus B\\ 
& \text{and any subset of } \bm{Pa}(B)\setminus A \text{, then } A-B \text{ is present.} \\  
\end{aligned}
\end{equation}
The following null and alternative are therefore equivalent to \eqref{undirected_HT}, where CI oracles are queried about $A$ and $B$ given all possible subsets of $\bm{Pa}(A)\setminus B$ and all possible subsets of $\bm{Pa}(B)\setminus A$:
\begin{equation} \label{parents1}
\begin{aligned}
H_0: & \text{ At least one CI oracle outputs independent},\\
H_1: & \text{ All CI oracles output dependent}.
\end{aligned}
\end{equation}
Notice that the above hypothesis test is the same as the hypothesis test in \eqref{max_bound_CI}. We can therefore bound the p-value of \eqref{undirected_HT} using:
\begin{equation} \label{undir_bound}
p_{A-B}^\prime \triangleq \max_{i=1,\dots,q'}{p_{A\ci B|\mathbf{R}_i}},
\end{equation}
where $\mathbf{R}_i \subseteq \{\textbf{Pa}(A)\setminus B \}$ or $\mathbf{R}_i \subseteq \{ \textbf{Pa}(B)\setminus A \}$ and $q'$ denotes the total number of such subsets. 

Note that the skeleton discovery phase of the PC algorithm cannot differentiate between the parents and children of a particular vertex using its neighbors. However, we can further bound \eqref{undir_bound} using the following quantity:

\begin{equation} \label{undir_bound2}
p_{A-B}^\prime \leq \max_{i=1,\dots,q}{p_{A\ci B|\mathbf{S}_i}} \triangleq p_{A-B},
\end{equation}
where $\mathbf{S}_i \subseteq \{\textbf{N}(A)\setminus B \}$ or $\mathbf{S}_i \subseteq \{ \textbf{N}(B)\setminus A \}$ and $q$ denotes the total number of such subsets. 

Now assume that the Type II error rate of all CI tests is zero. Then, if the alternative holds for the CI tests (conditional dependence), then the alternative is accepted. Hence, the PC algorithm will not remove any of the edges between $\bm{N}(A)$ and $A$ as well as any of the edges between $\bm{N}(B)$ and $B$. PC therefore performs all necessary CI tests for computing \eqref{undir_bound2}, so upper bounding the Type I error rate for \eqref{undirected_HT} reduces to taking the maximum of the p-values for all of the CI tests performed by PC regarding $A$ and $B$. For example, suppose we measure three random variables $A$, $B$ and $C$. Then we obtain p-values after the PC algorithm tests whether $A \ci B$ and $A \ci B|C$. Suppose these p-values are $(0.03, 0.04)$ so that the PC algorithm with an $\alpha$ threshold of 0.05 determines that $A-B$ is present. The p-value upper bound of \eqref{undirected_HT} thus corresponds to $\max{\{0.03,0.04\}}=0.04$.

\subsection{Detecting V-Structures}
\subsubsection{Deterministic Skeleton}
The hypothesis testing procedure for directed edges is more complicated than the procedure for adjacencies. Edges can be oriented in the PC algorithm according to unshielded v-structures or the orientation rules as described in Section \ref{sec_2_2}. Let us first focus on the former and, for further simplicity, let us also assume that 1) we have access to the ground truth skeleton and 2) no edge is involved in more than one unshielded v-structure (we will later drop these assumptions in Section \ref{sec_4_2_2}). Our task then is to statistically infer the presence of an unshielded v-structure.

We now present the following null and alternative for each unshielded v-structure after finding a triple $A-C-B$ such that $A$ and $B$ are non-adjacent in the skeleton:
\begin{equation} \label{first_v_struc}
\begin{aligned}
H_0: &\text{ Unshielded $A \to C \leftarrow B$ is absent},\\
H_1: &\text{ Unshielded $A \to C \leftarrow B$ is present}.
\end{aligned}
\end{equation}
Next, consider the following proposition:
\begin{proposition} \label{thm_adj2}
\citep{Spirtes00} Consider the same assumptions as Proposition \ref{thm_adj}. Further assume that $A,C$ are adjacent and $C,B$ are adjacent but $A,B$ are non-adjacent. Then, A and B are conditionally independent given some subset of $\bm{Pa}(A)\setminus B$ which does not include $C$ or some subset of $\bm{Pa}(B)\setminus A$ which does not include $C$ if and only if $A \to C \leftarrow B$.
\end{proposition}
%\begin{proof}
%First notice that $\bm{Pa}(A)=\{\bm{Pa}(A)\setminus B\}$ and $\bm{Pa}(B)=\{\bm{Pa}(B)\setminus A\}$, since $A$ and $B$ are non-adjacent. As a result, we can instead prove that the if and only if statement holds for $\bm{Pa}(A)$ and $\bm{Pa}(B)$ without loss of generality.

%For the forward direction, suppose $A$ and $B$ are conditionally independent given some subset of $\bm{Pa}(A)$ which does not include $C$ or some subset of $\bm{Pa}(B)$ which does not include $C$. Then $A$ and $B$ are d-separated given some subset of $\bm{Pa}(A)$ which does not include $C$ or some subset of $\bm{Pa}(B)$ which does not include $C$ by d-separation faithfulness. Now if $C \in \bm{Pa}(A)$ or $C \in \bm{Pa}(B)$ (or both), then $A$ and $B$ are d-connected given a set which does not include $C$. Thus, $C \not \in \bm{Pa}(A)$ and $C \not \in \bm{Pa}(B)$, so $C$ must be a child of $A$ and a child of $B$. We conclude that $A \to C \leftarrow B$.

%For the other direction, if $A \to C \leftarrow B$ holds and $A$ and $B$ are non-adjacent, then $A$ and $B$ are d-separated given some subset of $\bm{Pa}(A)$ which does not include $C$ or some subset of $\bm{Pa}(B)$ which does not include $C$ by Proposition \ref{thm_adj}. The global directed Markov property then implies that $A$ and $B$ are conditionally independent given some subset of $\bm{Pa}(A)$ which does not include $C$ or some subset of $\bm{Pa}(B)$ which does not include $C$.
%\end{proof}
The following null and alternative is therefore equivalent to \eqref{first_v_struc}:
\begin{equation} \label{first_v_struc_alt}
\begin{aligned}
H_0: &\text{ $A$ and $B$ are conditionally dependent given any subset of $\bm{Pa}(A)\setminus B$} \\ &\text{  which does not include $C$ and any subset of $\bm{Pa}(B)\setminus A$ which} \\ &\text{  does not include $C$,} \\
H_1: &\text{ $A$ and $B$ are conditionally independent given some subset of $\bm{Pa}(A)\setminus B$} \\ &\text{ which does not include $C$ or some subset of $\bm{Pa}(B)\setminus A$ which} \\ &\text{  does not include $C$.}
\end{aligned}
\end{equation}
The above alternative is reminiscent of the way in which PC determines the presence of an unshielded v-structure according to Algorithm \ref{pc_2} in the Appendix; specifically, if $C$ is not in the set which renders $A$ and $B$ conditionally independent, then $C$ in $A-C-B$ must be a collider. We however cannot bound the p-value of \eqref{first_v_struc_alt} using CI tests, because conditional dependence is in the null and conditional independence is in the alternative, as opposed to vice versa. As a result, we also consider the following proposition:
\begin{proposition}
Consider the same assumptions as Proposition \ref{thm_adj2}. Then, $A$ and $B$ are conditionally dependent given any subset of $\bm{Pa}(A)\setminus B$ containing $C$ and any subset of $\bm{Pa}(B)\setminus A$ containing $C$ if and only if $A \to C \leftarrow B$.
\end{proposition}
\begin{proof}
First notice that $\bm{Pa}(A)=\{\bm{Pa}(A)\setminus B\}$ and $\bm{Pa}(B)=\{\bm{Pa}(B)\setminus A\}$, since $A$ and $B$ are non-adjacent. As a result, we can instead prove that the if and only if statement holds for $\bm{Pa}(A)$ and $\bm{Pa}(B)$ without loss of generality.

For the forward direction, suppose $A$ and $B$ are conditionally dependent given any subset of $\bm{Pa}(A)$ containing $C$ and any subset of $\bm{Pa}(B)$ containing $C$. Then $A$ and $B$ are d-connected given any subset of $\bm{Pa}(A)$ containing $C$ and any subset of $\bm{Pa}(B)$ containing $C$ by the global directed Markov property. Clearly, $C\in \bm{N}(A)$ and $C\in \bm{N}(B)$, so $C$ must either be a parent of $A$ and a parent of $B$, a child of $A$ and a parent of $B$, a parent of $A$ and a child $B$, or a child of $A$ and a child of $B$. Note that $A$ and $B$ are non-adjacent, so $A$ and $B$ are d-separated given some subset of $\bm{Pa}(A)$ or some subset of $\bm{Pa}(B)$ by Proposition \ref{thm_adj} and d-separation faithfulness. Moreover, the subset must include $C$ if $C$ is a parent of $A$ and a parent of $B$, a child of $A$ and a parent of $B$, or a parent of $A$ and a child $B$; otherwise, $A$ and $B$ would be d-connected. As a result, in those three situations, we arrive at the contradiction that $A$ and $B$ are d-separated given some subset of $\bm{Pa}(A)$ containing $C$ or some subset of $\bm{Pa}(B)$ containing $C$. We conclude that $C$ must be a child of $A$ and a child of $B$.

For the other direction, if $A\to C\leftarrow B$ holds, then $A$ and $B$ are d-connected given any subset of $\bm{Pa}(A)$ containing $C$ and any subset of $\bm{Pa}(B)$ containing $C$. D-separation faithfulness then implies that $A$ and $B$ are conditionally dependent given any subset of $\bm{Pa}(A)$ containing $C$ and any subset of $\bm{Pa}(B)$ containing $C$.
\end{proof}
We can thus equivalently write \eqref{first_v_struc_alt} as:
\begin{equation} \label{first_v_struc_alt2}
\begin{aligned}
H_0: & \text{ $A$ and $B$ are conditionally independent given some subset of $\bm{Pa}(A)\setminus B$} \\ &\text{ containing $C$ or some subset of $\bm{Pa}(B)\setminus A$ containing $C$},\\
H_1: & \text{ $A$ and $B$ are conditionally dependent given any subset of $\bm{Pa}(A)\setminus B$} \\ &\text{ containing $C$ and any subset of $\bm{Pa}(B)\setminus A$ containing $C$}.
\end{aligned}
\end{equation}
We can bound the Type I error rate of the above hypothesis test by taking the maximum p-value over certain CI tests:
\begin{equation} \nonumber
p_{\gamma_{AB|C}}'\triangleq \max_{i=1,..,m^\prime} p_{A\ci B|\bm{M}_i },
\end{equation}
where $\bm{M}_i$ denotes a subset of $\bm{Pa}(A)\setminus B$ containing $C$ or a subset of $\bm{Pa}(B)\setminus A$ containing $C$, and $m^\prime$ is the total number of subsets $\bm{M}_i$. Of course, in practice, we do not know which vertices are the parents. However, we can also upper bound \eqref{first_v_struc_alt2} as follows:
\begin{equation} \label{first_v_struc_bound}
p_{\gamma_{AB|C}}'\leq \max_{i=1,..,m} p_{A\ci B|\bm{T}_i } \triangleq p_{\gamma_{AB|C} },
\end{equation}
where $\bm{T}_i$ denotes a subset of $\bm{N}(A)\setminus B$ containing $C$ or a subset of $\bm{N}(B)\setminus A$ containing $C$, and $m$ denotes the total number of subsets $\bm{T}_i$. Note that we do not need the zero Type II error rate assumption for computing \eqref{first_v_struc_bound}, since we assume that the skeleton is provided.

\subsubsection{Inferred Skeleton} \label{sec_4_2_2}
We have considered orienting the colliders, if we have access to the ground truth skeleton. We now consider the more complex problem of orienting the colliders, if we must also statistically infer the skeleton. 

We again consider the following null and alternative:
\begin{equation} \label{v_struc_inf}
\begin{aligned}
H_0: &\text{ Unshielded $A \to C \leftarrow B$ is absent},\\
H_1: &\text{ Unshielded $A \to C \leftarrow B$ is present}.
\end{aligned}
\end{equation}
Now, the PC algorithm determines that the alternative holds, if all of the following conditions are true:
\begin{equation} \label{v_struc_conds}
\begin{aligned}
1. \hspace{3mm} & A \text{ and } C \text{ are conditionally dependent given any subset of } \bm{Pa}(A)\setminus C \\ & \text{and any subset of } \bm{Pa}(C)\setminus A. \\
2. \hspace{3mm} & B \text{ and } C \text{ are conditionally dependent given any subset of } \bm{Pa}(B)\setminus C \\ & \text{and any subset of } \bm{Pa}(C)\setminus B. \\
3. \hspace{3mm} & A \text{ and } B \text{ are conditionally dependent given any subset of } \bm{Pa}(A)\setminus B \\ & \text{containing $C$ and any subset of } \bm{Pa}(B)\setminus A \text{ containing $C$}.
\end{aligned}
\end{equation}
We therefore have the following equivalent form of the null and alternative as in \eqref{v_struc_inf}, if we assume $A$ and $B$ are non-adjacent:
\begin{equation} \label{v_struc_conds_hyp}
\begin{aligned}
H_0: &\text{ At least one condition from \eqref{v_struc_conds} does not hold}, \\
H_1: &\text{ All conditions from \eqref{v_struc_conds} hold}.
\end{aligned}
\end{equation}
Note that the non-adjacency assumption is reasonable because we did not have enough statistical evidence to invalidate the assumption when we executed \eqref{undirected_HT}. Indeed, non-adjacencies are always assumed unless the data suggests that the null of \eqref{undirected_HT} is unlikely. Now, the alternative of \eqref{v_struc_conds_hyp} is a series of three logical conjunctions, and the null is a series of three logical disjunctions as in \eqref{max_bound_gen}, so the Type I error rate of \eqref{v_struc_conds_hyp} can be bounded using the intersection bound:
\begin{equation} \label{v_bound_cond}
\begin{aligned}
\text{Pr(Conditions 1, 2, 3 $|H_0$)} &\leq \text{Pr(Any one condition $|H_0$)} \\
&\leq \max\{h_1,h_2,h_3\}.
\end{aligned}
\end{equation}
We will be using shorthand from here on. We write \eqref{v_struc_conds_hyp} equivalently as:
\begin{equation} \label{v_struc_symb}
\begin{aligned}
H_0: &\text{ }\neg(A-C) \vee \neg (B-C) \vee \neg \gamma_{AB|C}, \\
H_1: &\text{ }(A-C)\wedge (B-C)\wedge \gamma_{AB|C},
\end{aligned}
\end{equation}
where $A-C$, $B-C$, and $\gamma_{AB|C}$ represent Condition 1, 2 and 3 from \eqref{v_struc_conds}, respectively. We therefore have a p-value bound of \eqref{v_struc_inf} similar to \eqref{v_bound_cond}:
%\begin{equation} \label{one_v_bound_f}
\begin{align*} \label{one_v_bound_f}
&\hspace{5mm} \text{Pr}\Big((A-C)\wedge (B-C)\wedge \gamma_{AB|C} |H_0 \Big) \\ \nonumber
&\leq \text{Pr}\Big(A-C\big|\neg(A-C)\Big)\\ \nonumber
&\leq \max\Big\{\text{Pr}\Big(A-C\big|\neg(A-C)\Big), \text{Pr}\Big(B-C\big|\neg(B-C)\Big), \text{Pr}\Big(\gamma_{AB|C}\big|\neg\gamma_{AB|C}\Big)\Big\} \\ \nonumber
&\leq \max\{p_{A-C},p_{B-C},p_{\gamma_{AB|C} } \}. \numberthis
\end{align*}
%\end{equation}
Notice that computing $p_{\gamma_{AB|C} }$ requires $\bm{N}(A)$ and $\bm{N}(B)$, not just their respective empirical estimates $\widehat{\bm{N}}(A)$ and $\widehat{\bm{N}}(B)$ which PC can discover. However, we can invoke a zero Type II error rate assumption in order to ensure that $\bm{N}(A)\subseteq \widehat{\bm{N}}(A)$ and $\bm{N}(B)\subseteq \widehat{\bm{N}}(B)$ as explained in detail in Section 4.4, so $p_{\gamma_{AB|C} }$ can still be upper bounded. The assumption also ensures that we can upper bound $p_{A-C}$ and $p_{B-C}$ according to Section 4.1. We conclude that a zero Type II error rate ensures that \eqref{one_v_bound_f} can be computed.

\begin{figure}[b!]
\centering
\begin{tikzpicture} [
            > = stealth, % arrow head style
            auto,
            semithick % line style
        ]
    \node[shape=circle,draw=black] (A) at (0,0) {$A$};
    \node[shape=circle,draw=black] (C) at (0,-1.2) {$C$};
    \node[shape=circle,draw=black] (B1) at (-1.2,0) {$\bm{B}_1$};
    \node[shape=circle,draw=black] (B2) at (1.2,0) {$\bm{B}_2$};

    \path [->] (A) edge node[left] {}(C);
    \path [->] (B1) edge node[left] {}(C);
    \path [->] (B2) edge node[left] {}(C);
 
\end{tikzpicture}
\caption{Here, one can orient the edge $A-C$ according to the two unshielded v-structures $A \to C \leftarrow \bm{B}_1$ and $A \to C \leftarrow \bm{B}_2$.}
\label{fig_v_struc_both}
\end{figure}
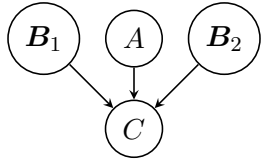

Next, consider the situation where PC can orient any one edge by using more than one unshielded v-structure. For example, consider the DAG in Figure \ref{fig_v_struc_both}. In this case, PC can orient $A-C$ by using either $\bm{B}_1\to C$ or $\bm{B}_2\to C$ (or both); we may therefore want to take both situations into account. Note that the original PC algorithm always orients an edge according to one v-structure which it picks arbitrarily according to the ordering of its computations. We thus only require the bound \eqref{one_v_bound_f} in this case. However, we will propose a modified PC algorithm in Section \ref{sec_5} which takes into account all possible ways to orient one edge. Now, we can use the following null and alternative for Figure \ref{fig_v_struc_both} when assuming that both $A$ and $\bm{B}_1$ and $A$ and $\bm{B}_2$ are non-adjacent:
\begin{equation} \label{mult_v_bound}
\begin{aligned}
H_0: &\text{ } \neg(A-C)\vee\Big([\neg(\bm{B}_1-C)\vee\neg\gamma_{A\bm{B}_1 |C} ]\wedge[\neg(\mathbf{B}_2-C)\vee\neg\gamma_{A\bm{B}_2 |C} ]\Big), \\
H_1: &\text{ } (A-C)\wedge\Big([(\bm{B}_1-C)\wedge\gamma_{A\bm{B}_1 |C} ]\vee[(\bm{B}_2-C)\wedge\gamma_{A\bm{B}_2 |C} ]\Big).
\end{aligned}
\end{equation}
We can therefore bound the Type I error rate of \eqref{mult_v_bound} as follows, where $\mathcal{G}=\neg(A-C)$ and $\mathcal{H}=[\neg(\bm{B}_1-C)\vee\neg\gamma_{A\bm{B}_1 |C} ]\wedge[\neg(\bm{B}_2-C)\vee\neg\gamma_{A\bm{B}_2 |C} ]$, $\mathcal{H}_1=\neg(\bm{B}_1-C)\vee\neg\gamma_{A\bm{B}_1 |C} $, and $\mathcal{H}_2=\neg(\bm{B}_2-C)\vee\neg\gamma_{A\bm{B}_2 |C} $:
\begin{equation}
\begin{aligned}
&\hspace{5mm}\text{Pr}\Big((A-C)\wedge \Big([(\bm{B}_1-C)\wedge\gamma_{A\bm{B}_1 |C} ]\vee[(\bm{B}_2-C)\wedge\gamma_{A\bm{B}_2 |C} ]\Big)\big|H_0 \Big) \\
&\leq \max\Big\{\text{Pr}(A-C|\mathcal{G} ),\text{Pr}\Big(\Big[(\bm{B}_1-C)\wedge\gamma_{A\bm{B}_1 |C} \Big]\vee\Big[(\bm{B}_2-C)\wedge\gamma_{A\bm{B}_2|C} \Big]\big|\mathcal{H} \Big) \Big\} \\
&\leq \max\Big\{\text{Pr}(A-C|\mathcal{G} ),\text{Pr}\Big((\bm{B}_1-C)\wedge\gamma_{A\bm{B}_1 |C}|\mathcal{H} \Big)+\text{Pr}\Big((\bm{B}_2-C)\wedge\gamma_{A\bm{B}_2|C} |\mathcal{H} \Big) \Big\} \\
&= \max\Big\{\text{Pr}(A-C|\mathcal{G} ),\text{Pr}\Big((\bm{B}_1-C)\wedge\gamma_{A\bm{B}_1 |C}|\mathcal{H}_1 \Big)+\text{Pr}\Big((\bm{B}_2-C)\wedge\gamma_{A\bm{B}_2|C} |\mathcal{H}_2 \Big) \Big\} \\
&\leq \max\Big\{\text{Pr}(A-C|\mathcal{G} ),\max\Big\{\text{Pr}\Big(\bm{B}_1-C \big| \neg(\bm{B}_1-C) \Big),\text{Pr}(\gamma_{A\bm{B}_1 |C} \big| \neg\gamma_{A\bm{B}_1 |C} )  \Big\} \\ & \hspace{4.5mm}+\max\Big\{\text{Pr}\Big(\bm{B}_2-C \big| \neg(\bm{B}_2-C) \Big),\text{Pr}(\gamma_{A\bm{B}_2 |C} \big| \neg\gamma_{A\bm{B}_2 |C} )  \Big\} \Big\} \\
&\leq \max \Big\{p_{A-C},\max\{p_{\bm{B}_1-C},p_{\gamma_{A\bm{B}_1 |C} } \}+\max\{p_{\bm{B}_2-C},p_{\gamma_{A\bm{B}_2 |C} } \} \Big\}.
\end{aligned}
\end{equation}
More generally, for an arbitrary number, say $j$, of multiple possible ways to orient $A-C$ by unshielded v-structures, we have:
%\begin{equation} \label{bound_f_v}
\begin{align*} \label{bound_f_v}
&\text{Pr}\Big((A-C)\wedge\Big([(\bm{B}_1-C)\wedge\gamma_{A\bm{B}_1 |C} ]\vee ... \vee[(\bm{B}_j-C)\wedge\gamma_{A\bm{B}_j |C} ]\Big)\big|H_0 \Big) \\ \nonumber
&\leq \max\Big\{p_{A-C},\sum_{i=1}^j \max\{p_{\bm{B}_i-C},p_{\gamma_{A\bm{B}_i |C} } \} \Big\}, \numberthis
\end{align*}
%\end{equation}
assuming that $\bm{B}_1,\dots,\bm{B}_j$ are all non-adjacent to $A$.

\subsection{Orientation Rules}

We now consider bounding the p-values of edges which are oriented using the orientation rules of the PC algorithm. Recall from \eqref{rules1} that the PC algorithm only requires the repeated application of three orientation rules to be complete. We analyze these three orientation rules in separate subsections.

\subsubsection{First orientation rule}
We can construct the hypothesis test for the first orientation rule as follows according to the sufficient conditions of the first rule in \eqref{rules1}:
\begin{equation} \label{rule_one_test}
\begin{aligned}
H_0: &\text{ }\neg(A-B)\vee\neg(C\to A), \\
H_1: &\text{ }(A-B)\wedge(C\to A).
\end{aligned}
\end{equation}
We again also assume that $C$ and $B$ are non-adjacent. Now the PC algorithm determines that the alternative holds, if all of the following conditions are true:
\begin{enumerate}
\item $A-B:$  $A$ and $B$ are conditionally dependent given any subset of $\bm{Pa}(A)\setminus B$ and any subset of $\bm{Pa}(B)\setminus A$.
\item $C\to A:$ An edge is oriented from $C$ to $A$ under two scenarios. In the first, the edge is oriented because $A$ is the collider in an unshielded v-structure. In the second, the edge is oriented due to the previous application of an orientation rule.
\end{enumerate}
We thus have a logical conjunction and can bound the Type I error rate using the intersection bound:
\begin{equation} \nonumber
\text{Pr}((A-B)\wedge(C\to A)|H_0 )\leq \max\{p_{A-B},p_{C\to A} \},
\end{equation}
where $p_{C\to A}$ refers to the p-value bound for the hypothesis test of an unshielded v-structure or a previously applied orientation rule. Of course, $p_{C\to A}$ will be the former when the PC algorithm begins to execute the orientation rules. More generally, for $\bm{C}_i \to A$ that can orient $A-B$ where $i=1,...,j$, we have:
%\begin{equation} \label{bound_o1}
\begin{align*} \label{bound_o1}
\text{Pr}\Big((A-B)\wedge \Big[(\bm{C}_1\to B)\vee\dots\vee(\bm{C}_j\to A)\Big]|H_0 \Big) \\ \nonumber
\leq \max\Big\{p_{A-B},\sum_{i=1}^j p_{\bm{C}_i\to A} \Big\}, \numberthis
\end{align*}
%\end{equation}
where we require that $\bm{C}_1,\dots,\bm{C}_j$ are all non-adjacent to $B$.

\subsubsection{Second orientation rule}
We have the following hypothesis test according to the sufficient conditions of the second rule in (1):
\begin{equation} \label{rule_two_test}
\begin{aligned}
H_0: &\text{ }\neg (A-B)\vee \neg(A\to C\to B), \\
H_1: &\text{ }(A-B)\wedge (A\to C\to B).
\end{aligned}
\end{equation}
Hence, by conjunction:
\begin{equation} \nonumber
\text{Pr}((A-B)\wedge (A\to C\to B)|H_0 )\leq \max\{p_{A-B},p_{A\to C\to B} \},
\end{equation}
where  $p_{A\to C\to B}\leq\max\{p_{A\to C},p_{C\to B}\}$. The above Type I error rate can therefore be further upper bounded by $\max\{p_{A-B},p_{A\to C},p_{C\to B} \}$. More generally, we have:
%\begin{equation} \label{bound_o2}
\begin{align*} \label{bound_o2}
\text{Pr}\Big((A-B)\wedge \Big[(A\to \bm{C}_1\to B)\vee ...\vee(A\to \bm{C}_j\to B)\Big]|H_0 \Big) \\ \nonumber
\leq \max\Big\{p_{A-B},\sum_{i=1}^j p_{A\to \bm{C}_i\to B} \Big\}\leq\max\Big\{p_{A-B},\sum_{i=1}^j \max\{p_{A\to \bm{C}_i },p_{\bm{C}_i\to B}\} \Big\}. \numberthis
\end{align*}
%\end{equation}

\subsubsection{Third orientation rule}
We have the following null and alternative by the sufficient conditions of the third rule in \eqref{rules1}, assuming that $C$ and $D$ are non-adjacent:
\begin{equation} \label{rule_three_test}
\begin{aligned}
H_0: &\text{ } \neg(A-B)\vee\neg(A-C\to B)\vee\neg(A-D\to B), \\
H_1: &\text{ } (A-B)\wedge(A-C\to B)\wedge(A-D\to B).
\end{aligned}
\end{equation}
We can bound the Type I error rate of the above hypothesis test as follows:
\begin{equation} \nonumber
\begin{aligned}
&\hspace{5mm} \text{Pr}\Big((A-B)\wedge (A-C\to B)\wedge (A-D\to B)|H_0 \Big) \\
&\leq \max\{p_{A-B},p_{A-C\to B},p_{A-D\to B} \} \\
&\leq\max\Big\{p_{A-B},\max\{p_{A-C},p_{C\to B} \},\max\{p_{A-D},p_{D\to B}\}\Big\}.
\end{aligned}
\end{equation}
The general case is slightly more complicated than the first and second orientation rules. In this case, we need to control the Type I error rate of accepting at least two paths as opposed to one. Let the set $\bm{D}$ include all three-node paths from $A$ to $B$ with the first edge undirected from $A$ to a middle vertex and the second edge directed from the middle vertex to $B$ such that the $i^{\text{th}}$ element of $\bm{D}$ is:
\begin{equation} \nonumber
\bm{D}_i\triangleq {A-\bm{C}_i\to B}.
\end{equation}
Let us suppose $\bm{D}$ has a total of $n$ elements and assume that no middle vertex $\bm{C}_i$ is adjacent to any other middle vertex. Now, let $\bm{D}'$ be the set containing all of the $n$ choose 2 elements of $\bm{D}$. The $i^{\text{th}}$ element in $\bm{D}'$ is therefore:
\begin{equation} \nonumber
\bm{D}_i'\triangleq \{A-\bm{C}_k\to B,A-\bm{C}_l\to B\},
\end{equation}
where $k$ and $l$ are the distinct indices represented the two chosen middle vertices. Let $\bm{D}_{i,1}'$ and $\bm{D}_{i,2}'$ be the first and second elements in $\bm{D}_i'$, respectively. Also let $r={n \choose 2}$. We then have:
\begin{equation} \nonumber
\begin{aligned}
&\hspace{5mm} \text{Pr}\Big((A-B)\wedge\Big[ (\bm{D}_{1,1}'\wedge \bm{D}_{1,2}' ) \vee ... \vee (\bm{D}_{r,1}'\wedge \bm{D}_{r,2}' )\Big]\big|H_0 \Big)\\
&\leq \max\Big\{p_{ A-B},\sum_{i=1}^r\text{Pr}(\bm{D}_i' |H_0)\Big\},\\
\end{aligned}
\end{equation}
where Pr$(\bm{D}_i' | H_0 )\triangleq p_{\{A-\bm{C}_k\to B,A-\bm{C}_l\to B\}}$ and is bounded as follows:
\begin{equation} \nonumber
\begin{aligned}
\hspace{5mm} \text{Pr}(\bm{D}_i' | H_0 ) &\leq \max\{p_{A-\bm{C}_k\to B},p_{A-\bm{C}_l\to B} \}\\
&\leq \max\{p_{A-\bm{C}_k },p_{\bm{C}_k\to B},p_{A-\bm{C}_l },p_{\bm{C}_l\to B}\}\\ &\triangleq \text{\"{Pr}} (\bm{D}_i' | H_0 ),
\end{aligned}
\end{equation}
We therefore have:
%\begin{equation} \label{bound_o3}
\begin{align*} \label{bound_o3}
&\hspace{5mm} \text{Pr}\Big((A-B)\wedge{ \Big[ (\bm{D}_{1,1}'\wedge \bm{D}_{1,2}' ) \vee ... \vee (\bm{D}_{r,1}' \wedge \bm{D}_{r,2}' )\Big]}\big|H_0 \Big) \\ \nonumber
&\leq \max\Big\{p_{A-B},\sum_{i=1}^r \text{\"{Pr}} ̈⁡(\bm{D}_i' | H_0 ) \Big\}. \numberthis
\end{align*}
%\end{equation}

\subsection{Summary and Analysis of the Bounds} \label{sec_4_4}

We derived several bounds for edge orientation as summarized in Table 1. We created the bounds by engineering specific hypothesis tests and successively applying the union and intersection bounds accordingly. Note that $j$ and $r$ are usually very small in sparse graphs. 

\begin{table*}[b!]
\caption{P-value bounds for all of the edge types in a CPDAG. Note that $\bm{S}_i \subseteq \{\bm{N}(A)\setminus B \}$ or $\bm{S}_i \subseteq \{\bm{N}(B)\setminus A \} $.}  \label{table_bounds}
\begin{center}
\begin{tabular}{|c c c|}
\hline
& & \\[-6pt]
\textbf{Edge Type} & \textbf{P-Value Bound} & \textbf{Equation Num.} \\[5pt]
\hline
& & \\[-6pt]
Undirected & $\max_i{p_{A\ci B|\bm{S}_i}}$ & \eqref{undir_bound2}\\[7pt]
Unshielded v-structure & $\max\Big\{p_{A-C},\sum_{i=1}^j \max\{p_{\bm{B}_i-C},p_{\gamma_{A\bm{B}_i |C} } \} \Big\}$ & \eqref{bound_f_v}\\[7pt]
First orientation rule & $\max\Big\{p_{A-B},\sum_{i=1}^j p_{\bm{C}_i\to A} \Big\}$ & \eqref{bound_o1}\\[7pt]
Second orientation rule & $\max\Big\{p_{A-B},\sum_{i=1}^j \max\{p_{A\to \bm{C}_i },p_{\bm{C}_i\to B}\} \Big\}$ & \eqref{bound_o2}\\ [7pt]
Third orientation rule & $\max\Big\{p_{A-B},\sum_{i=1}^r \text{\"{Pr}} ̈⁡(\bm{D}_i' | H_0 ) \Big\}$ & \eqref{bound_o3}\\ [7pt]
\hline
\end{tabular}
\end{center}
\end{table*}

One may now wonder whether PC can actually control the bounds listed in Table \ref{table_bounds} (we say that a quantity can be \textit{controlled}, if the quantity can be upper bounded). Recall that we provided a rough, affirmative answer to the question in Sections \ref{sec_4_1} and \ref{sec_4_2_2} by assuming a zero Type II error rate. We now spell out a more detailed answer via a theorem whose proof builds on the argument of Theorem 4 in \citep{Armen14}.
\begin{theorem} \label{theorem_type2}
Suppose that the PC algorithm is applied to a sample from $\mathbb{P}$ represented by DAG $\mathbb{G}$. If we have:
\begin{enumerate}
\item $\mathbb{P}$ is d-separation faithful to $\mathbb{G}$,
\item The Type II error rate is zero,
\item The PC algorithm also tests whether any two non-adjacent vertices $A$, $B$ with common neighbor $C$ are conditionally dependent given any subset of $\bm{Pa}(A)\setminus B$ containing $C$ and any subset of $\bm{Pa}(B)\setminus A$ containing $C$,
\end{enumerate}
then all of the p-value bounds in Table \ref{table_bounds} can be controlled using the p-values of the CI tests executed by PC.
\end{theorem}
\begin{proof} 
Consider any two vertices $A$ and $B$. Algorithm \ref{pcp_1} starts with a fully connected graph, so we have $B \in \widehat{\bm{N}}(A)$ and $A \in \widehat{\bm{N}}(B)$ in the beginning. Note that Algorithm \ref{pcp_1} executes $test_{A\ci B|\bm{S}}$ for all $\bm{S}\subseteq \widehat{\bm{N}}(A)\setminus B$ and for all $\bm{S} \subseteq \widehat{\bm{N}}(B)\setminus A$. The zero Type II error rate ensures the following: if the alternative holds, then the alternative is accepted. As a result, Algorithm \ref{pcp_1} will not remove any vertices adjacent to $A$ and any vertices adjacent to $B$ with a zero Type II error rate. Hence, we always have $\{\bm{N}(A)\setminus B\}\subseteq \{\widehat{\bm{N}} ̂(A)\setminus B\}$ and $\{\bm{N}(B)\setminus A\}\subseteq \{\widehat{\bm{N}} ̂(B)\setminus A\}$. Algorithm \ref{pcp_1} therefore must eventually execute $test_{A\ci B|\bm{S}}$ for all $\bm{S}\subseteq \{\bm{N}(A)\setminus B\}$ and for all $\bm{S}\subseteq \{\bm{N}(B)\setminus A$\}, so \eqref{undir_bound2} can be controlled.

For \eqref{bound_f_v}, the p-value bounds for undirected edges can already be controlled by the previous paragraph. We must now argue that $p_{\gamma_{AB|C} }$ can be controlled. Let $C$ be a collider between non-adjacent vertices $A$ and $B$. Now notice that $C \in \bm{N}(A)\subseteq \widehat{\bm{N}}(A)$ and $C \in \bm{N}(B) \subseteq \widehat{\bm{N}}(B)$, so Algorithm \ref{pcp_2} must execute $test_{A\ci B|\bm{S}}$ for all $\bm{S}\subseteq \{\bm{N}(A)\setminus B\}$ containing $C$ and for all $\bm{S}\subseteq \{\bm{N}(B)\setminus A\}$ containing $C$. Hence $p_{\gamma_{AB|C}}$ can be controlled.

Now the p-value bounds \eqref{bound_o1}, \eqref{bound_o2} and \eqref{bound_o3} can be controlled trivially because the p-value bounds for \eqref{undir_bound} and \eqref{bound_f_v} can be controlled. 
\end{proof}
In other words, PC can control the bounds in Table 1 with some additional CI tests and a zero Type II error rate. 

Of course, the Type II error rate is never zero in practice, but this becomes less of an issue as the sample size increases. We may also consider reducing the Type II error rate by simultaneously implementing  three strategies:
\begin{enumerate}
\item Use a liberal (higher) $\alpha$ threshold. We for example often use an $\alpha$ threshold of 0.20 in the experiments. This is the simplest strategy which decreases the Type II error rate but also increases the Type I error rate. However, we can then control the Type I error rate post-hoc with an FDR controlling procedure. Of course, setting the $\alpha$ threshold too high will prevent the PC algorithm from terminating within a reasonable amount of time as well as loosen the p-value bounds, since the CI tests will fail to explain away many edges. We therefore cannot rely entirely on this first strategy.
\item Use hypothesis tests whose p-value bounds are robust to Type II errors. The hypothesis tests in Section \ref{sec_4_4} are in fact robust to such errors due to the intersection bound as explained in detail in Appendix \ref{appendix_2}. Briefly, we can also reasonably consider modifying the null hypotheses of \eqref{v_struc_symb}, \eqref{rule_one_test}, \eqref{rule_two_test} and \eqref{rule_three_test} to ``no edges between any of the vertices." This corresponds to converting the logical disjunctions in the null of \eqref{max_bound_gen} into conjunctions which in turns leads to a less robust p-value bound involving the minimum of a set of p-values instead of the maximum. As a result, under-estimating one p-value in the p-value set due to Type II error(s) can cause PC to also under-estimate the bound of \eqref{v_struc_symb}, \eqref{rule_one_test}, \eqref{rule_two_test} or \eqref{rule_three_test}. 
\item Modify the PC algorithm to prevent and catch many Type II errors.
\end{enumerate}
The last strategy is more complex, so we discuss it in detail in the next section.

\section{The PC Algorithm with P-Values} \label{sec_5}
We now propose a modified PC algorithm called PC with p-values (PC-p) that reduces the influence of Type II errors by preventing and catching potential Type II errors. At the same time, PC-p is correct - the algorithm operates differently than PC, but it maintains PC's desirable soundness and completeness properties.

The PC-p algorithm involves two ideas. First, PC-p performs skeleton discovery with the same skeleton discovery procedure used in the PC-stable algorithm \citep{Colombo14}. This procedure ensures that the algorithm does not skip some CI tests due to Type II errors and variable ordering. The second idea behind PC-p involves a modification to the procedure for propagating edge orientations. Specifically, if two edge orientations conflict, PC-p admits bidirected edges instead of over-writing previous orientations like PC. PC-p then unorients the bidirected edges as well as the directed edges which were directly used to infer the presence of the bidirected edges. The algorithm subsequently labels the resulting undirected edges as ``ambiguous'' which ensures that PC-p does not orient additional edges using the ambiguous edges. Indeed, the PC-p algorithm uses conflicts in edge orientation to detect potential Type II errors and prevent the propagation of the errors throughout the graph. In practice, we find that these two modifications to the PC algorithm help PC-p with the BY estimator achieve more accurate strong estimation and control of the FDR than PC, as we will see in Section \ref{sec_6}.

We now describe the PC-p algorithm in detail; however, we will not describe the computation of the p-value upper bounds until Section \ref{sec_p_val_comp} in order to keep the presentation clear. We have divided the PC-p algorithm into Algorithms \ref{pcp_1}, \ref{pcp_2}, \ref{pcp_3} and  \ref{pcp_4}, where the first three procedures correspond to Algorithms \ref{pc_1}, \ref{pc_2} and \ref{pc_3} of the original PC algorithm. 

\subsection{Skeleton Discovery}
\begin{figure}[b!]
\centering 
\begin{tikzpicture}[
            > = stealth, % arrow head style
            auto,
            semithick % line style
        ]
	\node[shape=circle] at (-0.3,1.8) {(a)};
    \node[shape=circle,draw=black] (A) at (0,0) {A};
    \node[shape=circle,draw=black] (E) at (1.5,0) {E};
    \node[shape=circle,draw=black] (C) at (1.5,1.5) {C};
    \node[shape=circle,draw=black] (D) at (1.5,-1.5) {D};
    \node[shape=circle,draw=black] (B) at (3,0) {B};

    \path [->] (A) edge node[left] {}(E);
    \path [->] (A) edge node[left] {}(C);
    \path [->] (A) edge node[left] {}(D);
    \path [->] (C) edge node[left] {}(E);
    \path [->] (D) edge node[left] {}(E);
    \path [->] (B) edge node[left] {}(C);
    \path [->] (B) edge node[left] {}(E);
    \path [->,bend left] (C) edge node[left] {}(D);
    
	\node[shape=circle] at (4.2,1.8) {(b)};
    \node[shape=circle,draw=black] (A) at (4.5,0) {A};
    \node[shape=circle,draw=black] (E) at (6,0) {E};
    \node[shape=circle,draw=black] (C) at (6,1.5) {C};
    \node[shape=circle,draw=black] (D) at (6,-1.5) {D};
    \node[shape=circle,draw=black] (B) at (7.5,0) {B};
    
    \path [-] (A) edge node[left] {}(E);
    \path [-] (A) edge node[left] {}(C);
    \path [-] (A) edge node[left] {}(D);
    \path [-] (C) edge node[left] {}(E);
    \path [-] (D) edge node[left] {}(E);
    \path [-] (B) edge node[left] {}(C);
    \path [-] (B) edge node[left] {}(E);
    
	\node[shape=circle] at (8.7,1.8) {(c)};
    \node[shape=circle,draw=black] (A) at (9,0) {A};
    \node[shape=circle,draw=black] (E) at (10.5,0) {E};
    \node[shape=circle,draw=black] (C) at (10.5,1.5) {C};
    \node[shape=circle,draw=black] (D) at (10.5,-1.5) {D};
    \node[shape=circle,draw=black] (B) at (12,0) {B};
    
    \path [-] (A) edge node[left] {}(E);
    \path [-] (A) edge node[left] {}(C);
    \path [-] (A) edge node[left] {}(D);
    \path [-] (C) edge node[left] {}(E);
    \path [-] (D) edge node[left] {}(E);
    \path [-] (B) edge node[left] {}(C);
    \path [-] (B) edge node[left] {}(E);
    \path [-] (B) edge node[left] {}(D);
\end{tikzpicture}
\caption{An example of a situation when PC infers different skeletons due to a Type II error and two variable orderings. (a) The true causal graph, (b) the skeleton inferred by PC from $order_1(\bm{X})$, (c) the skeleton inferred by PC from $order_2(\bm{X})$.}
\label{fig_skeleton}
\end{figure}
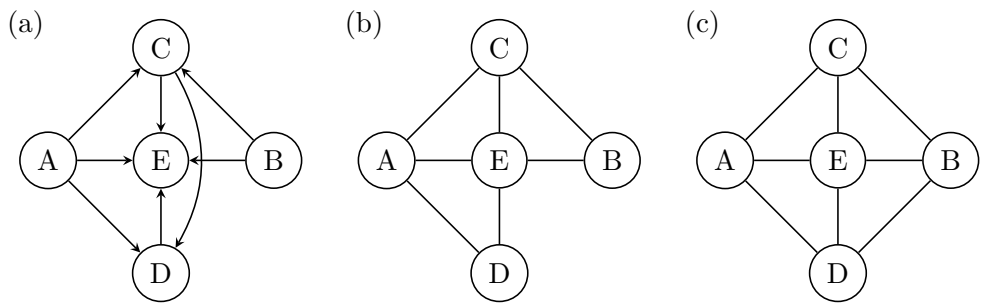

We first consider skeleton discovery. The original PC algorithm uses Algorithm \ref{pc_1} to discover the skeleton. However, Algorithm \ref{pc_1} can cause the sample version of the PC algorithm to skip some CI tests due to variable ordering and Type II errors. For example, consider the causal graph in Figure \hyperref[fig_skeleton]{4a} as first presented in \citep{Colombo14}. In this example, suppose the CI tests correctly determine that $A\ci B$ and $B\ci D|\{A,C\}$ but incorrectly determine that $C\ci D|\{A,E\}$. The incorrect inference is a Type II error, since $C$ and $D$ are adjacent in the true graph. Now consider the following ordering of variables for the PC algorithm: $order_1 (\bm{X})=(A,D,B,C,E)$. In this case, the ordered pair $(D,B)$ is considered before $(D,C)$ in Algorithm \ref{pc_1}, since $(D,B)$ comes earlier in $order_1 (\bm{X})$. The PC algorithm removes $D-B$ because a CI test determines that $D\ci B|\{A,C\}$ and $\{A,C\}$ is a subset of $\bm{N}(D)=\{A,B,C,E\}$. Next, $D-C$ is considered and erroneously removed because a CI test determines that $D\ci C|\{A,E\}$ and $\{A,E\}$ is a subset of $\bm{N}(D)=\{A,C,E\}$. We thus ultimately obtain the skeleton in Figure \hyperref[fig_skeleton]{4b} with $order_1 (\bm{X})$.

Now consider an alternative ordering of the variables: $order_2 (\bm{X})=(A,C,D,B,E)$. In this case, $(C,D)$ is considered before $(D,B)$ in Algorithm \ref{pc_1}, and the algorithm erroneously removes $C-D$. Next, the algorithm considers $D-B$ but $\{A,C\}$ is not a subset of $\bm{N}(D)=\{A,B,E\}$, so $D-B$ remains. Even when the PC algorithm eventually also considers the same undirected edge as $B-D$, $\{A,C\}$ is again not a subset of $\bm{N}(B)=\{C,D,E\}$, so $B-D$ remains. In other words, $(C,D)$ is considered first in $order_2 (\bm{X})$ which causes $C$ to be removed from $\bm{N}(D)$. Algorithm \ref{pc_1} therefore never executes $test_{B\ci D|\{A,C\}}$. We thus ultimately obtain the skeleton in Figure \hyperref[fig_skeleton]{4c} with $order_2 (\bm{X})$. 

The previous two examples show that the Type II error of incorrectly determining that $C\ci D|\{A,E\}$ leads PC to infer two different skeletons due to differences in variable ordering. Clearly, we would like to eliminate the dependency of skeleton discovery on variable ordering and also reduce its dependency on Type II errors at the same time. Fortunately, Colombo and Maathius proposed such a modification of Algorithm \ref{pc_1} as outlined in Algorithm \ref{pcp_1}. The key difference between Algorithm \ref{pc_1} and \ref{pcp_1} involves the for loop in steps \ref{pcp_1:adj_1}-\ref{pcp_1:adj_3} of Algorithm \ref{pcp_1} which computes and stores the adjacency sets after each new conditioning set size. As a result, an incorrect edge deletion due to a Type II error on line \ref{pcp_1:delete_edge} of Algorithm \ref{pcp_1} no longer effects which CI tests are performed for other pairs of variables with conditioning set size $l$. Indeed, the algorithm only modifies the adjacency sets when it increases the conditioning set size. Colombo and Maathius proved that Algorithm \ref{pcp_1} is order-independent. We review the proof here, since it is informative:

\begin{proposition} \citep{Colombo14}. The skeleton resulting from Algorithm \ref{pcp_1} is order-independent.
\end{proposition}
\begin{proof}
Consider the removal or retention of some undirected edge $A-B$ at some conditioning set size $l$. The ordering of the variables determines the order in which the edges (line \ref{pcp_1:pair}) and subsets $\bm{S} \subseteq \bm{a}(A)$ and $\bm{S} \subseteq \bm{a}(B)$ (line \ref{pcp_1:subset}) are considered. However, by construction, the order in which the edges are considered does not affect the sets $\bm{a}(A)$ and $\bm{a}(B)$.

If there is at least one subset $\bm{S}$ of $\bm{a}(A)$ or $\bm{a}(B)$ such that $A\ci B| \bm{S}$, then any ordering of the variables will find a separating set for $A$ and $B$ (but different orderings may lead to different separating sets as illustrated in Example 2 of \citep{Colombo14}). Conversely, if there is no subset $\bm{S}'$ of $\bm{a}(A)$ or $\bm{a}(B)$ such that $A \ci B|\bm{S}'$, then no ordering will find a separating set.

Hence, any ordering of the variables leads to the same edge deletions and therefore to the same skeleton.
\end{proof}

In other words, modifying the adjacency sets only when changing the conditioning set size prevents PC-p from skipping some CI tests during skeleton discovery because of Type II errors and variable ordering. As a result, Algorithm \ref{pcp_1} enables PC-p to perform more of the required CI tests than Algorithm \ref{pc_1} in order to correctly upper bound the p-value of \eqref{undirected_HT}. However, notice that Algorithm \ref{pcp_1} does not prevent all Type II errors from effecting the skeleton. The edge $C-D$ is for example eliminated in Figure \ref{fig_skeleton} regardless of the ordering because of the erroneous conclusion that $C\ci D|\{A,E\}$. As a result, we have $C\not \in \widehat{\bm{N}}(D)$ which may lead to under-estimation of the p-value bounds for undirected edges connected to $D$. We will nonetheless see in Section \ref{sec_6} that Algorithm \ref{pcp_1} does help PC-p achieve tighter estimation and control of the FDR than the original skeleton discovery procedure, since Algorithm \ref{pcp_1} eliminates the influence of at least some Type II errors.

\subsection{Unshielded V-Structures} \label{sec_v_struc}
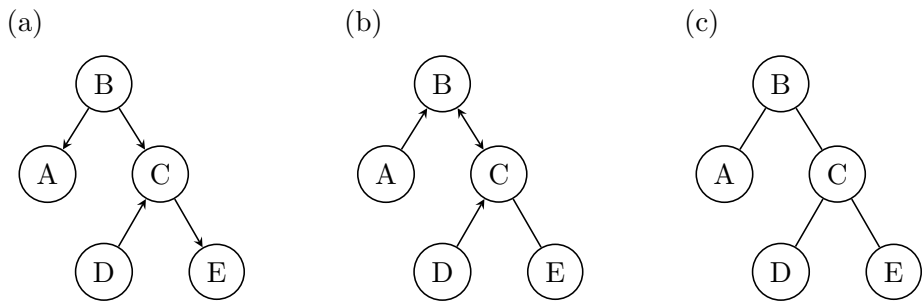
\begin{figure}[b!]
\centering
\begin{tikzpicture} [
            > = stealth, % arrow head style
            auto,
            semithick % line style
        ]
	\node[shape=circle] at (-0.3,1.8) {(a)};
    \node[shape=circle,draw=black] (A) at (0,-0.2) {A};
    \node[shape=circle,draw=black] (B) at (0.75,1) {B};
    \node[shape=circle,draw=black] (C) at (1.5,-0.2) {C};
    \node[shape=circle,draw=black] (E) at (2.25,-1.5) {E};
    \node[shape=circle,draw=black] (D) at (0.75,-1.5) {D};

    \path [->] (B) edge node[left] {}(A);
    \path [->] (B) edge node[left] {}(C);
    \path [->] (D) edge node[left] {}(C);
    \path [->] (C) edge node[left] {}(E);
    
	\node[shape=circle] at (4.2,1.8) {(b)};
       \node[shape=circle,draw=black] (A) at (4.5,-0.2) {A};
    \node[shape=circle,draw=black] (B) at (5.25,1) {B};
    \node[shape=circle,draw=black] (C) at (6,-0.2) {C};
    \node[shape=circle,draw=black] (E) at (6.75,-1.5) {E};
    \node[shape=circle,draw=black] (D) at (5.25,-1.5) {D};

    \path [->] (A) edge node[left] {}(B);
    \path [<->] (C) edge node[left] {}(B);
    \path [->] (D) edge node[left] {}(C);
    \path [-] (E) edge node[left] {}(C);

	\node[shape=circle] at (8.7,1.8) {(c)};
     \node[shape=circle,draw=black] (A) at (9,-0.2) {A};
    \node[shape=circle,draw=black] (B) at (9.75,1) {B};
    \node[shape=circle,draw=black] (C) at (10.5,-0.2) {C};
    \node[shape=circle,draw=black] (E) at (11.25,-1.5) {E};
    \node[shape=circle,draw=black] (D) at (9.75,-1.5) {D};

    \path [-] (A) edge node[left] {}(B);
    \path [-] (C) edge node[left] {}(B);
    \path [-] (D) edge node[left] {}(C);
    \path [-] (E) edge node[left] {}(C);
\end{tikzpicture}
\caption{Example of how Algorithm \ref{pcp_2} deals with conflicting edge orientations. a) The ground truth, b) the inferred graph with two v-structures $A \to B \leftarrow C$ and $D \to C \leftarrow B$ that lead to the bi-directed edge $B \leftrightarrow C$, and c) the final graph after unorienting both v-structures.}
\label{fig_v_struc}
\end{figure}
We now describe Algorithm \ref{pcp_2}, where we use the circle edge endpoint ``$\circ$'' as a meta-symbol representing either a tail or an arrowhead. In Algorithm \ref{pcp_2}, PC-p orients edges according to all unshielded v-structures in line \ref{pcp_2:v_struc_orient}, even if two v-structures conflict with each other in the direction of a particular edge. In the case of conflict, PC-p admits a bidirected edge instead of favoring one particular direction over the other. The algorithm then unorients all v-structures involving the bidirected edges and labels the unoriented edges as ``ambiguous'' in line \ref{pcp_2:unorient} because bidirected edges may result from a Type II error. For example, consider the ground truth in Figure \hyperref[fig_v_struc]{5a} and assume that Algorithm \ref{pcp_1} correctly discovers all of the undirected edges. Moreover, assume Algorithm \ref{pcp_1} correctly finds a separating set of $B$ and $D$ that does not contain $C$ but incorrectly finds a separating set of $A$ and $C$ that does not contain $B$. The latter is a Type II error, since the alternative should have been accepted rather than rejected when conditioning on a subset not containing $B$. In this case, PC-p first orients the edges according to Figure \hyperref[fig_v_struc]{5b}. However, notice that the two unshielded v-structures conflict with each other due to the bidirected edge $B \leftrightarrow C$, and PC-p cannot determine which v-structure admitted the Type II error. As a result, the algorithm unorients all of the edges in both v-structures as in Figure \hyperref[fig_v_struc]{5c}. PC-p then labels the three unoriented edges as ``ambiguous'' so that the algorithm does not orient any other undirected edges based on these three edges using the orientation rules. The labeling thus prevents the algorithm from propagating Type II errors by orienting additional edges based on the erroneous directions. 

\begin{algorithm}[] \label{pcp_1}
 \KwData{$\bm{X}^n$, $\alpha$}
 \KwResult{$\widehat{\mathbb{G}}$, $\mathcal{P}^1$, $\mathcal{S}$, $\mathcal{I}$}
 \BlankLine
 Form a completely connected undirected graph $\widehat{\mathbb{G}}$ on the variable set in $\bm{X}^n$\\
 $l=-1$ \\
 \Repeat{all ordered pairs of adjacent variables $(A,B)$ in $\widehat{\mathbb{G}}$ satisfy $|\bm{a}(A)\setminus B|\leq l$}{
 $l=l+1$ \\
 \For{each variable $A$ in $\widehat{\mathbb{G}}$}{ \label{pcp_1:adj_1}
 	$\bm{a}(A) \leftarrow \widehat{\bm{N}}(A)$ \label{pcp_1:adj_2}
 } \label{pcp_1:adj_3}
 \Repeat{all ordered pairs of adjacent variables $(A,B)$ in $\widehat{\mathbb{G}}$ with $|\bm{a}(A)\setminus B| \geq l$ have been considered}{
 Select a new ordered pair of variables $(A, B)$ that are adjacent in $\widehat{\mathbb{G}}$ and satisfy $|\bm{a}(A)\setminus B|\geq l$ \label{pcp_1:pair}\\
 \Repeat{$A - B$ is deleted from $\widehat{\mathbb{G}}$ or all $\bm{S} \subseteq \{\bm{a}(A)\setminus B \}$ with $|\bm{S}| =l$ have been considered}{Choose a new set $\bm{S} \subseteq \{\bm{a}(A)\setminus B\}$ with $|\bm{S}|=l$ \label{pcp_1:subset}\\
 $p \leftarrow$ p-value from $test_{A \ci B | \bm{S}}$\\
 \uIf{$p \leq \alpha$}{
 Insert $p$ into $\mathcal{P}_{AB}^1$ and $\mathcal{P}_{BA}^1$ \label{pcp_1:pval_1}}
 \Else{Delete $A - B$ from $\widehat{\mathbb{G}}$ \label{pcp_1:delete_edge}\\
 Empty $\mathcal{P}_{AB}^1$ and $\mathcal{P}_{BA}^1$ \\
 Insert $\bm{S}$ into $\mathcal{S}_{AB}$ and $\mathcal{S}_{BA}$ \\
 }
 }
}
}

 \For{each nonempty $\mathcal{P}_{AB}^1$ in $\mathcal{P}^1$}{
 $\mathcal{P}_{AB}^1 \leftarrow\max\{\mathcal{P}_{AB}^1\}$ \label{pcp_1:pval_2} \\
 Place the same unique identifier for $A-B$ into $\mathcal{I}_{AB}$ and $\mathcal{I}_{BA}$ \label{pcp_1:identifier}
 }

 \caption{Skeleton Discovery}
\end{algorithm}

\begin{algorithm}[]  \label{pcp_2}
 \KwData{$\bm{X}^n$, $\widehat{\mathbb{G}}$, $\mathcal{P}^1$, $\mathcal{S}$, $\mathcal{I}$}
 \KwResult{$\widehat{\mathbb{G}}$, $\mathcal{P}^2$, $\mathcal{I}$}
 \BlankLine
 \For{all ordered pairs of non-adjacent variables $(A, B)$ with common neighbor $C$}{
 \If{$C \not \in \text{ any set in } \mathcal{S}_{AB}$}{
 Replace $A \circlinecirc C \circlinecirc B$ with $A \circ \!\! {\to} \hspace{1mm}  C \leftarrow \!\! \circ \hspace{1mm}  B$ \label{pcp_2:v_struc_orient}\\
 $l=0$ \label{pcp_2:pval_1}\\
 \Repeat{all $\bm{S}\subseteq \widehat{\bm{N}}(A)$ including $C$ and all $\bm{S}\subseteq \widehat{\bm{N}}(B)$ including $C$ satisfy $|\bm{S}|\leq l$ \label{pcp_2:pval_3}}{
 $l=l+1$\\
 \Repeat{all $\bm{S}\subseteq \widehat{\bm{N}}(A)$ including $C$ and all $\bm{S}\subseteq \widehat{\bm{N}}(B)$ including $C$ with $|\bm{S}|=l$ have been considered}{
 Choose a new set $\bm{S}\subseteq \widehat{\bm{N}}(A)$ including $C$ or $\bm{S}\subseteq \widehat{\bm{N}}(B)$ including $C$ with $|\bm{S}|=l$\\
 $p \leftarrow$ p-value from $test_{A \ci B | \bm{S}}$ \\
 Insert $p$ into $\mathcal{P}''$ \label{pcp_2:pval_2}
 
 }
 
 }
 Insert $\max\{\mathcal{P}_{AC}^1, \mathcal{P}''\}$ into $\mathcal{P}'_{BC}$ \label{pcp_2:pval_4}\\
 Insert $\max\{\mathcal{P}_{BC}^1, \mathcal{P}''\}$ into $\mathcal{P}'_{AC}$ \label{pcp_2:pval_5}\\
 Empty $\mathcal{P}''$

 }

 }
 $\widehat{\mathbb{G}}' \leftarrow \widehat{\mathbb{G}}$ \\
\For{each $A \leftrightarrow C$ in $\widehat{\mathbb{G}}$ \label{pcp_2:pval_7}}{
 Unorient the edge to $A-C$ in $\widehat{\mathbb{G}}'$\\
 For each additional edge directed to $A$ in $\widehat{\mathbb{G}}$, unorient the edge in $\widehat{\mathbb{G}}'$ \\
 For each additional edge directed to $C$ in $\widehat{\mathbb{G}}$, unorient the edge in $\widehat{\mathbb{G}}'$ \\
 Label the unoriented edges as ``ambiguous'' in $\widehat{\mathbb{G}}'$ \label{pcp_2:unorient} \\
  }
  $\widehat{\mathbb{G}} \leftarrow \widehat{\mathbb{G}}'$ \\
 \For{each $A \rightarrow C$ in $\widehat{\mathbb{G}}$ \label{pcp_2:pval_6}}{
 $\mathcal{P}_{AC}^2 \leftarrow \max\Big\{\mathcal{P}_{AC}^1, \text{sum}[\mathcal{P}'_{AC}]\Big\}$ \label{pcp_2:pval_9}\\
 \uIf{one p-value in $\mathcal{P}'_{AC}$}{
 Place the same unique identifier into $\mathcal{I}_{AC}$ and $\mathcal{I}_{BC}$ for unshielded collider $A \to C \leftarrow B$ \label{pcp_2:identifier1}\\
 }
 \ElseIf{more than one p-value in $\mathcal{P}'_{AC}$}{Place a unique identifier into $\mathcal{I}_{AC}$ \label{pcp_2:identifier2}\\}
 
 }

 \caption{Unshielded V-structures}
\end{algorithm}

\begin{algorithm}[] \label{pcp_3}
 \KwData{$\widehat{\mathbb{G}}$, $\mathcal{P}^1$, $\mathcal{P}^2$, $\mathcal{I}$}
 \KwResult{$\widehat{\mathbb{G}}$, $\mathcal{P}^2$, $\mathcal{I}$}
 \BlankLine
 \Repeat{there are no more edges to orient}{
 $\widehat{\mathbb{G}}' \leftarrow \widehat{\mathbb{G}}$ \\
 \If{$A - B$ non-ambiguous and $\exists i$ s.t. $\bm{C}_i \to A$ with $\bm{C}_i$ and $B$ non-adjacent in $\widehat{\mathbb{G}}$}{
 Replace $A \circlinecirc B$ in $\widehat{\mathbb{G}}'$ with $A \circ \!\! {\to} \hspace{1mm} B$\\
 $\mathcal{P}_{A,B}'\leftarrow \text{sum} \Big\{ \mathcal{P}_{A,B}', \text{sum}\{\mathcal{P}_{\bm{C}_i A}^2, \forall i \text{ s.t.  $\bm{C}_i \to A$ with $\bm{C}_i,B$ non-adjacent}\} \Big\}$ \label{pcp_3:pval_1}\\
 }
 \If{$A - B$ non-ambiguous and $\exists i$ s.t. $A \to \bm{C}_i \to B$ in $\widehat{\mathbb{G}}$}{
 Replace $A \circlinecirc B$ in $\widehat{\mathbb{G}}'$ with $A \circ \!\! {\to} \hspace{1mm} B$ \\
 $\mathcal{P}_{AB}' \leftarrow \text{sum}\Big\{\mathcal{P}_{AB}', \text{sum}\Big[\max\{\mathcal{P}_{A \bm{C}_i}^2,\mathcal{P}_{\bm{C}_i B}^2\}, \forall i \text{ s.t. } A \to \bm{C}_i \to B \Big]\Big\}$ \label{pcp_3:pval_2} \\
 }
 \If{$A - B$ non-ambiguous and $\exists i,j$ s.t. $A - \bm{C}_i \to B$, $A - \bm{C}_j \to B$ with $A - \bm{C}_i$ and $A - \bm{C}_j$ non-ambiguous, and $\bm{C}_i$ and $\bm{C}_j$ non-adjacent in $\widehat{\mathbb{G}}$}{
 Replace $A \circlinecirc B$ in $\widehat{\mathbb{G}}'$ with $A \circ \!\! {\to} \hspace{1mm} B$\\
 $\mathcal{P}_{AB}' \leftarrow \text{sum}\Big\{\mathcal{P}_{AB}', \text{sum}\Big[\max\{\mathcal{P}_{A \bm{C}_i}^1,\mathcal{P}_{\bm{C}_i B}^2,\mathcal{P}_{A\bm{C}_j}^1,\mathcal{P}_{\bm{C}_jB}^2\}, \forall i,j \text{ s.t. } A - \bm{C}_i \to B, A - \bm{C}_j \to B$ with $A - \bm{C}_i$ and $A - \bm{C}_j$ non-ambiguous, and $\bm{C}_i$ and $\bm{C}_j$ non-adjacent $\Big]\Big\}$ \label{pcp_3:pval_3}\\
 } 
 \For{each $A \leftrightarrow B$ in $\widehat{\mathbb{G}}'$}{
 Unorient to $A - B$ in $\widehat{\mathbb{G}}'$ \label{pcp_3:unorient_bi}\\
 For each edge in the sufficient conditions of each orientation rule that led to the creation of a directed edge to $A$ in $\widehat{\mathbb{G}}$, unorient the edge in $\widehat{\mathbb{G}}'$\\
 For each edge in the sufficient conditions of each orientation rule that led to the creation of a directed edge to $B$ in $\widehat{\mathbb{G}}$, unorient the edge in $\widehat{\mathbb{G}}'$ \label{pcp_3:unorient_suff}\\
 Label the unoriented edges as ``ambiguous'' in $\widehat{\mathbb{G}}'$ \label{pcp_3:unorient}\\
 }
 $\widehat{\mathbb{G}} \leftarrow \widehat{\mathbb{G}}'$ \\

 \For{each $A \rightarrow B$ in $\widehat{\mathbb{G}}$}{
 Place a unique identifier into $\mathcal{I}_{AB}$ \label{pcp_3:identifier}\\
 $\mathcal{P}_{AB}^2\leftarrow\max\{\mathcal{P}_{AB}^1,\mathcal{P}_{AB}'\}$ \label{pcp_3:pval_4}\\
 }
   Empty $\mathcal{P}'$ \\
 }
 \For{each non-empty cell $\mathcal{P}_{AB}^1$ s.t. $\mathcal{P}_{AB}^2$ and $\mathcal{P}_{BA}^2$ are empty \label{pcp_3:transfer1}}{
 $\mathcal{P}_{AB}^2, \mathcal{P}_{BA}^2 \leftarrow \mathcal{P}_{AB}^1$
 }\label{pcp_3:transfer2}
 \caption{Orientation Rules}
\end{algorithm}

\begin{algorithm}[] \label{pcp_4}
 \KwData{$\widehat{\mathbb{G}}$, $\mathcal{P}^2$, $\mathcal{I}$, $\alpha$, $q$}
 \KwResult{$\widehat{FDR}_{BY}(\alpha)$, $\widehat{\mathbb{G}}^*$}
 \BlankLine
 \tcp{Estimation}
 $\widehat{FDR}_{BY}(\alpha) \leftarrow$ Solution of \ref{FDR_BY} using threshold $\alpha$ and $m$ corresponding to the number of unique identifiers in $\mathcal{I}$ \\
 \BlankLine
 \tcp{Control}
 $\alpha^* \leftarrow$ Solution of \ref{alpha_star} using FDR level $q$ and $m$ corresponding to the number of unique identifiers in $\mathcal{I}$ \\
 $\widehat{\mathbb{G}}^* \leftarrow$ $\widehat{\mathbb{G}}$ with edges associated with p-values above $\alpha^*$ eliminated\\
 \caption{FDR Estimation and Control}
\end{algorithm}

\subsection{Orientation Rules}

\begin{figure}[b!]
\centering 
\begin{tikzpicture}[
            > = stealth, % arrow head style
            auto,
            semithick % line style
        ]
    \node[shape=circle,draw=black] (A) at (0,0) {A};
    \node[shape=circle,draw=black] (B) at (1,-1) {B};
    \node[shape=circle,draw=black] (C) at (-0.75, -1) {C};
    \node[shape=circle,draw=black] (D) at (-1.5, 0) {D};
    \node[shape=circle,draw=black] (E) at (1.75,0) {E};
    \node[shape=circle,draw=black] (F) at (2.5,-1) {F};

    \path [-] (A) edge node[left] {}(B);
    \path [->] (D) edge node[left] {}(A);
    \path [->] (C) edge node[left] {}(A);
    \path [->] (E) edge node[left] {}(B);
    \path [->] (F) edge node[left] {}(B);
    
\end{tikzpicture}
\caption{Here, a bidirected edge between $A$ and $B$ results from the application of rule 1. PC-p therefore unorients and labels all edges in the above graph as ``ambiguous'' according to the sufficient conditions of rule 1.}
\label{fig_rule1_ambig}
\end{figure}
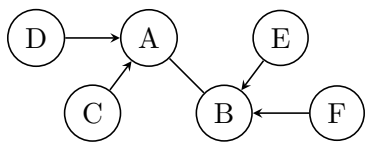

Notice that Algorithm \ref{pc_3} uses ``else if'' statements instead of all ``if'' statements. The ``else if'' approach is of course faster, but it also causes PC to ignore any interactions between the orientation rules in the sense that, if one rule orients an edge, then no other rule can orient an edge. PC-p performs the orientation rules according Algorithm \ref{pcp_3} which uses the ``if'' approach to attempt to apply all three orientation rules to each non-ambiguous undirected edge. Now, if bidirected edges exist after the rules are applied, then Algorithm \ref{pcp_3} unorients the edge as well as all edges involved in the sufficient conditions of the associated oientation rules in lines \ref{pcp_3:unorient_bi}-\ref{pcp_3:unorient_suff}. The algorithm then labels the unoriented edges as ``ambiguous'' in line \ref{pcp_3:unorient} similar to unshielded v-structure orientation in Section \ref{sec_v_struc}. For example, in Figure \ref{fig_rule1_ambig}, rule 1 of PC-p induces a bidirected edge between $A-B$, so PC-p unorients and labels all directed edges which satisfy the sufficient conditions of rule 1 as ambiguous; these include $D \rightarrow A, C \rightarrow A, E \rightarrow B,$ and $F \rightarrow B$.

\subsection{Analysis of PC-p}

We now have the following analysis of the PC-p algorithm:
\begin{theorem}
The PC-p algorithm with a CI oracle is sound and complete.
\end{theorem}
\begin{proof} PC is sound and complete, so it is enough to prove that PC-p and PC will perform the exact same edge deletion and edge orientation operations with a CI oracle. Note that Algorithm \ref{pcp_1} has already been shown to be sound and complete up to skeleton discovery (Colombo \& Maathius 2014). Algorithm \ref{pcp_1} will therefore perform the exact same edge deletions as Algorithm \ref{pc_1} with a CI oracle. Now, Algorithm \ref{pcp_2} will also perform the same edge orientations as Algorithm \ref{pc_2} with a CI oracle, since there will never be conflicting edge orientations. Lastly, for Algorithm \ref{pcp_3}, if there exists an edge that can be oriented by more than rule, then the edge must be oriented in the same direction by the other two rules. Algorithm \ref{pcp_3} therefore returns the same edge orientations as Algorithm \ref{pc_3}. We have proved equivalence in outputs of Algorithms \ref{pcp_1}, \ref{pcp_2} and \ref{pcp_3} of PC-p to Algorithms \ref{pc_1}, \ref{pc_2} and \ref{pc_3} of PC, respectively. Algorithm \ref{pcp_4} is not involved in graph structure discovery. \end{proof}
The output of PC-p is therefore equivalent to the output of PC in the large sample limit with a consistent CI test, even though PC-p performs more operations than PC.

\subsection{Computation of the P-Values} \label{sec_p_val_comp}
\begin{figure}[b!]
\centering 
\begin{tikzpicture}[
            > = stealth, % arrow head style
            auto,
            semithick % line style
        ]
    \node[shape=circle,draw=black] (A) at (0,0) {A};
    \node[shape=circle,draw=black] (B) at (0,-2) {B};
    \node[shape=circle,draw=black] (C) at (-1, -1) {C};
    \node[shape=circle,draw=black] (D) at (1, -1) {D};
    \node[shape=circle,draw=black] (E) at (2,-1) {E};
    \node[shape=circle,draw=black] (F) at (3,0) {F};

    \path [-] (A) edge node[left] {}(B);
    \path [-] (A) edge node[left] {}(C);
    \path [-] (A) edge node[left] {}(D);
    \path [->, bend left] (A) edge node[left] {}(E);
    \path [->] (C) edge node[left] {}(B);
    \path [->] (D) edge node[left] {}(B);
    \path [->] (C) edge node[left] {}(B);
    \path [->] (F) edge node[left] {}(E);
    \path [->, bend left] (E) edge node[left] {}(B);
    
\end{tikzpicture}
\caption{An example of a situation where two of PC's orientation rules, specifically rules 2 and 3, can orient one undirected edge ($A-B$) in the same direction. In this case, PC oriented all of the currently directed edges using unshielded v-structures.}
\label{fig_two_rules}
\end{figure}
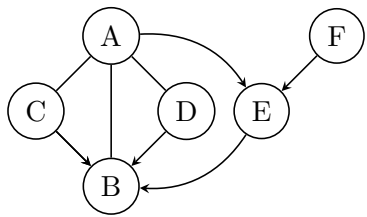

We now address the issue of computing the upper bounds of the p-values. Let us first consider Algorithm \ref{pcp_1}. Algorithm \ref{pcp_1} takes as input the dataset $\bm{X}^n$ and the significance threshold $\alpha$. The algorithm then stores the p-values of all significant CI tests in cell $\mathcal{P}^1$ when it reaches line \ref{pcp_1:pval_1}. Notice that the algorithm stores the p-values of all significant tests involving $A$ and $B$ in both $\mathcal{P}_{AB}^1$ and $\mathcal{P}_{BA}^1$. Algorithm 1 next computes the maximum over the p-values for all surviving edges in line \ref{pcp_1:pval_2} as in \eqref{undir_bound2}.

Algorithm \ref{pcp_2} takes $\mathcal{P}^1$ from Algorithm \ref{pcp_1} as input. Moreover, unlike Algorithm \ref{pc_2} of PC, Algorithm \ref{pcp_2} also takes as input the dataset $\bm{X}^n$, since PC-p must apply \eqref{bound_f_v} in order to obtain the upper bounds of the p-values for oriented unshielded v-structures. Indeed, Algorithm \ref{pcp_2} executes $test_{A\ci B|\bm{S}}$ for all $\bm{S}\subseteq \widehat{\bm{N}}(A)$ containing $C$ and all $\bm{S}\subseteq \widehat{\bm{N}}(B)$ containing $C$ in steps \ref{pcp_2:pval_1}-\ref{pcp_2:pval_3} for each $A-C-B$ such that $A$ and $B$ are non-adjacent and $C \not \in \mathcal{S}_{AB}$. Now, Algorithm \ref{pcp_2} ultimately stores all of the p-values needed to compute $p_{\gamma_{AB|C} }$ as in \eqref{first_v_struc_bound} in $\mathcal{P}''$ via line \ref{pcp_2:pval_2}. Algorithm \ref{pcp_2} then stores the maximum over $p_{A-C}$ and $p_{\gamma_{AB|C} }$ in $\mathcal{P}'_{BC}$ instead of $\mathcal{P}'_{AC}$ in line \ref{pcp_2:pval_4}. A similar set of operations eventually stores the maximum over $p_{B-C}$ and $p_{\gamma_{AB|C} }$ into $\mathcal{P}_{AC}'$ in line \ref{pcp_2:pval_5}. Note that multiple elements can enter into $\mathcal{P}_{BC}'$ and $\mathcal{P}_{AC}'$ when multiple v-structures can orient one edge. Finally, in line \ref{pcp_2:pval_9}, Algorithm \ref{pcp_2} takes the maximum over $\mathcal{P}_{AC}^1$ as returned from Algorithm \ref{pcp_1} and the sum of $\mathcal{P}_{AC}'$ to obtain $p_{A\to C}$ in $\mathcal{P}^2$ according to \eqref{bound_f_v} and similarly takes the maximum over $\mathcal{P}_{BC}^1$ and the sum of $\mathcal{P}_{BC}'$ to obtain $p_{B\to C}$ in $\mathcal{P}^2$.

Algorithm \ref{pcp_3} takes $\mathcal{P}^2$ from Algorithm \ref{pcp_2} as input. Next, in rule 1, Algorithm \ref{pcp_3} adds up the p-values associated with $\bm{C}_i \to A, \forall i$ and places the result in $\mathcal{P}_{AB}'$ in line \ref{pcp_3:pval_1} for computing \eqref{bound_o1}. Then, Algorithm \ref{pcp_3} sums over the maxima of $p_{A\to \bm{C}_i }$ and $p_{\bm{C}_i \to B}$ in rule 2 $\forall i$ s.t. $A\to \bm{C}_i \to B$ in line \ref{pcp_3:pval_2} for ultimately computing \eqref{bound_o2}. Subsequently, in rule 3, Algorithm \ref{pcp_3} finds all $n$ edges such that $A - \bm{C}_i \to B$. The algorithm then finds all of the $n$ choose 2 pairs, say $r$ of them. For each pair, say $A - \bm{C}_1\to B$ and $A - \bm{C}_2 \to B$, Algorithm \ref{pcp_3} computes $p_{A - \bm{C}_1\to B}$ and $p_{A - \bm{C}_2\to B}$ as the maximum over $p_{A - \bm{C}_1 }$ and $p_{\bm{C}_1\to B }$ and the maximum over $p_{A - \bm{C}_2 }$ and $p_{\bm{C}_2\to B }$, respectively. Algorithm \ref{pcp_3} next sums the p-values over all $r$ pairs in line \ref{pcp_3:pval_3} for computing \eqref{bound_o3}. Note that Algorithm \ref{pcp_3} also takes an outer-sum involving $\mathcal{P}_{AB}'$ in lines \ref{pcp_3:pval_1}, \ref{pcp_3:pval_2} and \ref{pcp_3:pval_3} of rules 1, 2 and 3, respectively; these summations correspond to logical disjunctions when multiple orientation rules can orient one edge in the same direction. For example, rules 2 and 3 can orient $A-B$ in the same direction in Figure \ref{fig_two_rules}. Two applications of rule 1 can also orient $A-B$ in the same direction in Figure \ref{fig_rule1_ambig}, if we remove one of the unshielded v-structures from the graph. Now, for all non-ambiguous edges, Algorithm \ref{pcp_3} then stores the maximum over the p-values from Algorithm \ref{pcp_1} and $\mathcal{P}'$ into $\mathcal{P}^2$ in line \ref{pcp_3:pval_4}. This process is repeated until no more edges can be oriented. Algorithm \ref{pcp_3} finally transfers the p-values of all of the remaining undirected edges in $\widehat{\mathbb{G}}$ from $\mathcal{P}^1$ to $\mathcal{P}^2$ in lines \ref{pcp_3:transfer1}- \ref{pcp_3:transfer2}. The algorithm therefore eventually outputs all of the final p-values in $\mathcal{P}^2$ as desired.

\subsection{Controlling the False Discovery Rate}

PC-p controls the FDR per hypothesis test as opposed to per edge, since the algorithm can sometimes orient two edges according to the same hypothesis test during unshielded v-structure discovery. Indeed, controlling the p-values per edge as opposed to per hypothesis test can result in overly conservative FDR estimation or control because an FDR estimator or controlling procedure may count the p-value of one hypothesis test multiple times.

PC-p keeps track of each distinct hypothesis test in Algorithms \ref{pcp_1}, \ref{pcp_2} and \ref{pcp_3} by using indexing cell $\mathcal{I}$ as follows. First, Algorithm \ref{pcp_1} assigns the same, unique identifier to the p-value bounds in both $\mathcal{P}_{AB}$ and $\mathcal{P}_{BA}$ in line \ref{pcp_1:identifier}. Next, if we have one v-structure $A\to C\leftarrow B$ that orients $A-C$ and $C-B$, then Algorithm \ref{pcp_2} associates both $A\to C$ and $C\leftarrow B$ with the same hypothesis test and therefore the same identifier in line \ref{pcp_2:identifier1}. On the other hand, if multiple unshielded v-structures can orient one edge, then Algorithm \ref{pcp_2} assigns the edge a unique identifier in line \ref{pcp_2:identifier2}, since a unique hypothesis test exists per edge in this case. Algorithm \ref{pcp_3} finally assigns a unique identifier to each newly oriented edge in line \ref{pcp_3:identifier} because each newly oriented edge also corresponds to a distinct hypothesis test. 

We can now use Algorithm \ref{pcp_4} to estimate and control the FDR using the identifiers in $\mathcal{I}$ and the p-value bounds in $\mathcal{P}^2$ as returned from Algorithm \ref{pcp_3}. Algorithm \ref{pcp_4} estimates the FDR by solving \ref{FDR_BY} to obtain $\widehat{FDR}_{BY}$, where $m$ corresponds to the number of unique identifiers in $\mathcal{I}$. The algorithm subsequently controls the FDR by solving \ref{alpha_star} to obtain $\alpha^*$. Algorithm \ref{pcp_4} then eliminates all edges with p-values below $\alpha^*$ in $\mathcal{P}^2$ in order to obtain $\widehat{\mathbb{G}}^*$; this process ensures that the expected FDR does not exceed $q$ in $\widehat{\mathbb{G}}^*$.

\subsection{Conclusion}

We wrap-up this section with the following theorem:
\begin{theorem}
Consider the same assumptions as Theorem \ref{theorem_type2}. Then PC-p achieves conservative point estimation and strong control of the FDR across the edges in $\widehat{\mathbb{G}}$.
\end{theorem}
\begin{proof}
We have already shown that PC-p can control the p-values of all of the edges in $\widehat{\mathbb{G}}$ from Theorem \ref{theorem_type2}. Estimation follows because the solution of \ref{FDR_BY} achieves conservative point estimation of the FDR at threshold $\alpha$ when the p-values are controlled \citep{Benjamini01}. Similarly, control follows because eliminating the edges associated with p-values above $\alpha^*$ as obtained from \ref{alpha_star} achieves strong control of the FDR at level $q$ when the p-values are in turn controlled \citep{Benjamini01}. 
\end{proof}
The PC-p algorithm thus corresponds to a valid method for estimating and controlling the FDR in the estimated CPDAG.

Note finally that PC-p takes slightly longer than original PC to complete because it performs extra computations. However, PC-p runs at approximately the same speed as PC-stable, since v-structure detection and orientation rule application take an infinitesimal amount of time compared to skeleton discovery.

\section{Experiments} \label{sec_6}
\subsection{Algorithms and Metrics} \label{sec_6_1}
We evaluated six algorithms:
\begin{enumerate}
\item PC-p,
\item PC-p without stabilization in the skeleton discovery procedure,
\item PC-p without ambiguous labelings during v-structure orientation and orientation rule application (PC-p without ambiguation),
\item PC-p without both stabilization and ambiguation,
\item PC-p without hypothesis tests with robust p-value bounds $-$ we chose the null hypotheses to be a series of logical conjunctions so that no edges are present between any of the variables. As a result, the p-values take on minimal values as described in Appendix \ref{appendix_2}. We call this procedure PC-p without robust p-values.
\item The original PC algorithm with p-value computation $-$ that is, we do not incorporate stabilization, and the algorithm arbitrarily over-writes edge orientations. We compute p-values according to the v-structure or rule which ultimately orients each edge in the CPDAG. The algorithm also performs some additional CI tests in order to compute \ref{first_v_struc_bound} as described in Section \ref{sec_4_1}.
\end{enumerate}
We ran these six algorithms because they are the only algorithms that allow us to compute the FDR across the entire CPDAG from the estimated p-values.

We assessed the FDR of the above six algorithms in detail using control and estimation bias\footnote{We also measured the false negative rate using the structural Hamming distance as a metric in Figure \ref{fig_appendix_SHdis} of the Appendix.}. An algorithm exhibits low control bias at FDR level $q$ when an FDR controlling procedure can accurately eliminate edges in the CPDAG using the p-values so that the FDR is in fact $q$. On the other hand, an algorithm exhibits low estimation bias when an FDR estimate closely matches the true FDR of the CPDAG. Notice that both control and estimation bias are important and can serve different purposes. As a result, we prefer an algorithm that exhibits both low control and estimation bias.

We used the mean of the following quantities to assess control bias:
\begin{equation}
\begin{aligned}
uc(\widehat{FDR}_{BY},q)&:=\max \{FDR ̂(\alpha^*)-q,0\}, \\
oc(\widehat{FDR}_{BY},q)&:=\max \{q-FDR ̂(\alpha^* ),0\},
\end{aligned}
\end{equation}
where $uc(\widehat{FDR}_{BY},q)$ denotes under-control at FDR level $q$ with the BY FDR estimate, and $oc(\widehat{FDR}_{BY},q)$ similarly denotes over-control. In the experiments, we varied $q$ from $[0.001,0.1]$ using 100 equispaced intervals. Note that we compute both under-control and over-control per CPDAG. A method achieves strong control when the mean under-control taken across the hypothesis tests is zero \citep{Armen14}. Moreover, the less the mean over-control, the tighter the strong control. As a result, achieving a lower mean under-control is more important than achieving a lower mean over-control. We therefore say that one method outperforms another if the method achieves a lower mean under-control while also maintaining a reasonably low mean over-control. 

We used the mean of the following similar quantities for estimation bias:
\begin{equation}
\begin{aligned}
ue(\widehat{FDR}_{BY},\alpha)&:=\max \{FDR(\alpha)-\widehat{FDR}_{BY} (\alpha),0\}, \\
oe(\widehat{FDR}_{BY},\alpha)&:=\max \{\widehat{FDR}_{BY} (\alpha)-FDR(\alpha),0\},
\end{aligned}
\end{equation}
where $ue(\widehat{FDR}_{BY},\alpha)$ denotes under-estimation at threshold level $\alpha$ with the BY FDR estimate, and $oe(\widehat{FDR}_{BY},\alpha)$ similar denotes over-estimation. We varied the $\alpha$ threshold from $[1\text{E-10},0.1]$ with 100 equispaced intervals in the experiments. Now, we say that estimation is conservative in a $\alpha$ threshold region when the underestimation is zero. Moreover, the greater the over-estimation in a p-value threshold region, the more conservative the estimate. A method should conservatively estimate the FDR but not do so over-conservatively. As a result, achieving lower under-estimation is more important than achieving lower over-estimation, and one method outperforms another if the method achieves a lower mean under-estimation while maintaining a reasonably low mean over-estimation.

Below, we report the relative performance differences of the six algorithms in recovering the CPDAG at a liberal $\alpha$ threshold of 0.20, since this threshold consistently provided a nice tradeoff between p-value bound looseness and low Type II error rates. We have reported the results using other $\alpha$ thresholds of 0.01, 0.05, 0.10, and 0.15 or 0.50 in Figures \ref{fig_appendix_lowD}-\ref{fig_appendix_GDP} of Appendix \ref{appendix_3}, with similar relative performance differences between the algorithms.  Figures \ref{fig_appendix_lowD_skel}-\ref{fig_appendix_real_skel} in the Appendix also contain results for skeleton discovery, where we compared the original skeleton discovery procedure of PC against the same procedure with stabilization. As expected, the stabilization procedure improved performance. We finally provide results with the more commonly used structural Hamming distance in \ref{fig_appendix_SHdis} of the Appendix; here, PC-p achieved superior performance by conservatively estimating the graph.

Note that for the simulations in Sections \ref{sec_6_2} and \ref{sec_6_3}, we generated the DAGs using the TETRAD V package (version 5.2.1) by drawing uniformly over all DAGs with a maximum in-degree of 2 and a maximum out-degree of 2. We then converted each of the DAGs to linear non-recursive SEM-IEs by 1) drawing the linear coefficients from independent standard normal distributions, and 2) setting independent Gaussian distributions over the error terms with standard deviations also drawn from the standard normal. Each linear SEM-IE with the error distributions therefore induced a multivariate Gaussian distribution across the observed variables. We finally ran all of the six algorithms using Fisher's z-test with a liberal $\alpha$ threshold of 0.20 and a maximum conditioning set size of 2. 

\subsection{Low Dimensional Inference} \label{sec_6_2}
\begin{figure}
\centering
    \includegraphics[width=1\textwidth]{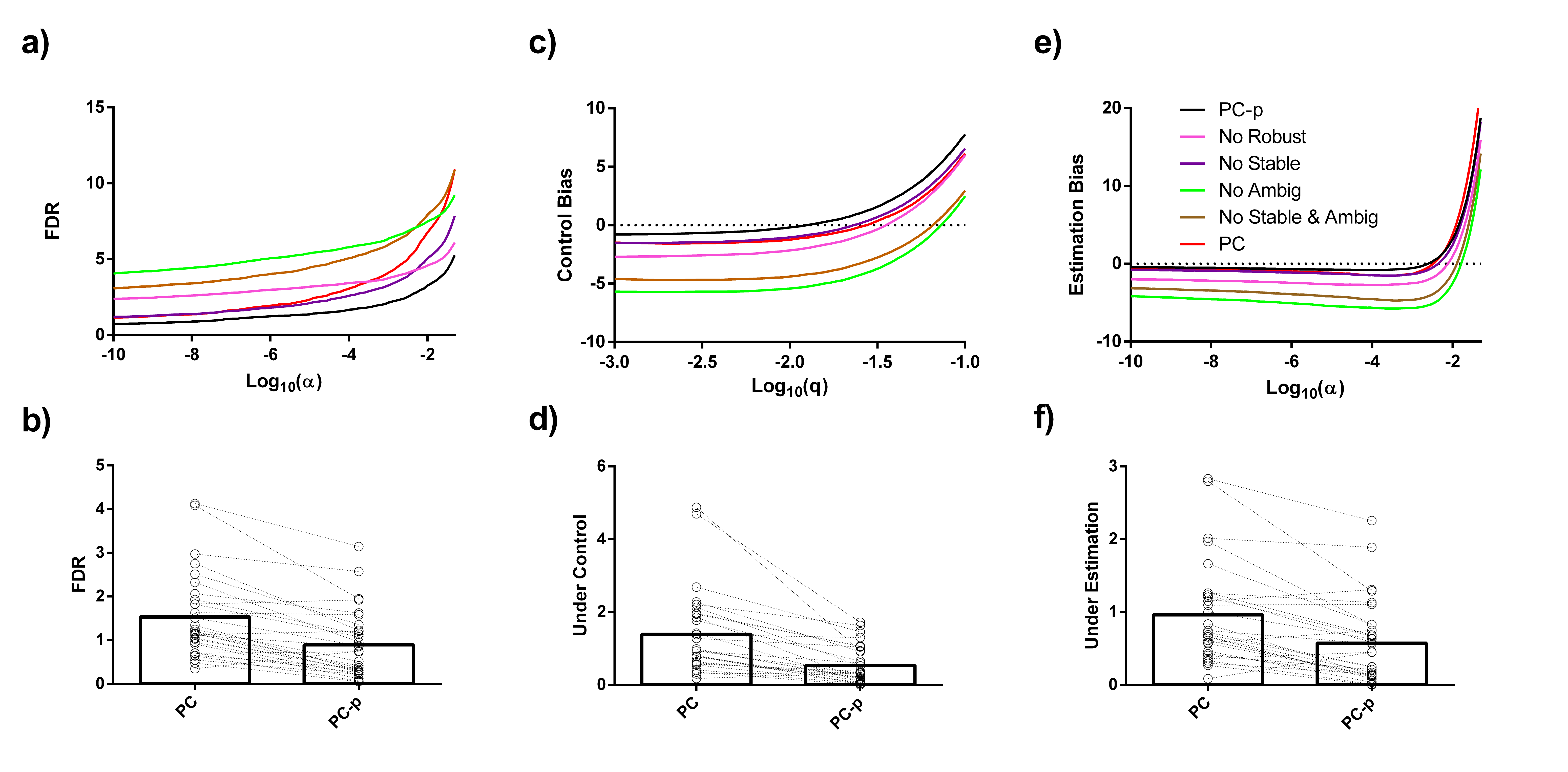}
\caption{Performances of PC-p, PC-p without robust p-values (no robust), PC-p without ambiguation (no ambig), PC-p without stabilization (no stable), PC-p without ambiguation and stabilization (no stable \& no ambig), and the original PC algorithm (PC) as assessed by (a,b) the FDR, (c,d) control bias, and (e,f) estimation bias in units of percent. PC-p achieved significantly lower FDR, under-estimation and under-control than the other five methods suggesting that robust p-values, stabilization and ambiguation are all important components of PC-p.}
 \label{fig_lowD}
\end{figure}
We generated 30 DAGs by drawing uniformly over all DAGs with 20 vertices. We converted each of the DAGs to 5 linear non-recursive SEM-IEs. We subsequently created 5 datasets using each linear SEM-IE with sample sizes of 100, 500, 1000, 5000, and 10000. We therefore created a total of $30\times 5\times 5=750$ datasets.

We analyzed the ability of the algorithms in correctly estimating the CPDAG in terms of the four metrics proposed in Section \ref{sec_6_1} as well as the FDR values. Results as averaged over DAGs, parameters and sample sizes are summarized in Figure \ref{fig_lowD}. We assessed the significance of all inter-algorithm differences using paired Wilcoxon signed rank tests. PC-p obtained lower mean FDR values than PC-p without robust p-values (Figures \hyperref[fig_lowD]{8a} and \hyperref[fig_lowD]{8b}; $z=\text{-}4.782, p=1.734\text{E-}6$), PC-p without stabilization ($z=\text{-}4.371, p=1.238\text{E-}5$), PC-p without ambiguation ($z=\text{-}4.782, p=1.734\text{E-}6$), PC-p without both stabilization and ambiguation ($z=\text{-}4.782, p=1.734\text{E-}6$), and PC original ($z=\text{-}4.433, p=9.316\text{E-}6$). Moreover, PC-p achieved significantly lower mean under-control than the competing methods (Figures \hyperref[fig_lowD]{8c} and \hyperref[fig_lowD]{8d}; vs. no robust: $z=\text{-}4.782, p=1.734E-6$; vs. no stable: $z=\text{-}3.898, p=9.711\text{E-}5$; vs. no ambig: $z=\text{-}4.782, p=1.734\text{E-}6$; vs. no stable \& no ambig: $z=-4.782, p=1.734\text{E-}6$; vs. PC original: $z=\text{-}4.700, p=2.603\text{E-}6$); meanwhile, PC-p kept the mean over-control small at $3.865\%$ (SD: $0.795\%$).

Results for estimation were similar. PC-p achieved significantly lower mean under-estimation than the competing methods (Figures \hyperref[fig_lowD]{8e} and \hyperref[fig_lowD]{8f}; vs. no robust: $z=\text{-}4.782, p=1.734\text{E-}6$; vs. no stable: $z=\text{-}4.206, p=2.597\text{E-}5$; vs. no ambig: $z=-4.782, p=1.734\text{E-}6$; vs. no stable \& no ambig: $z=\text{-}4.782, p=1.734\text{E-}6$; vs. PC original: $z=\text{-}4.186, p=2.843\text{E-}5$). PC-p also achieved a small degree of mean over-estimation ($8.222\%$, SD: $0.400\%$). We conclude that robust p-values, stabilization, and ambiguation all help PC-p achieve the lowest under-control and under-estimation.

\subsection{High Dimensional Inference} \label{sec_6_3}
\begin{figure}
\centering
    \includegraphics[width=1\textwidth]{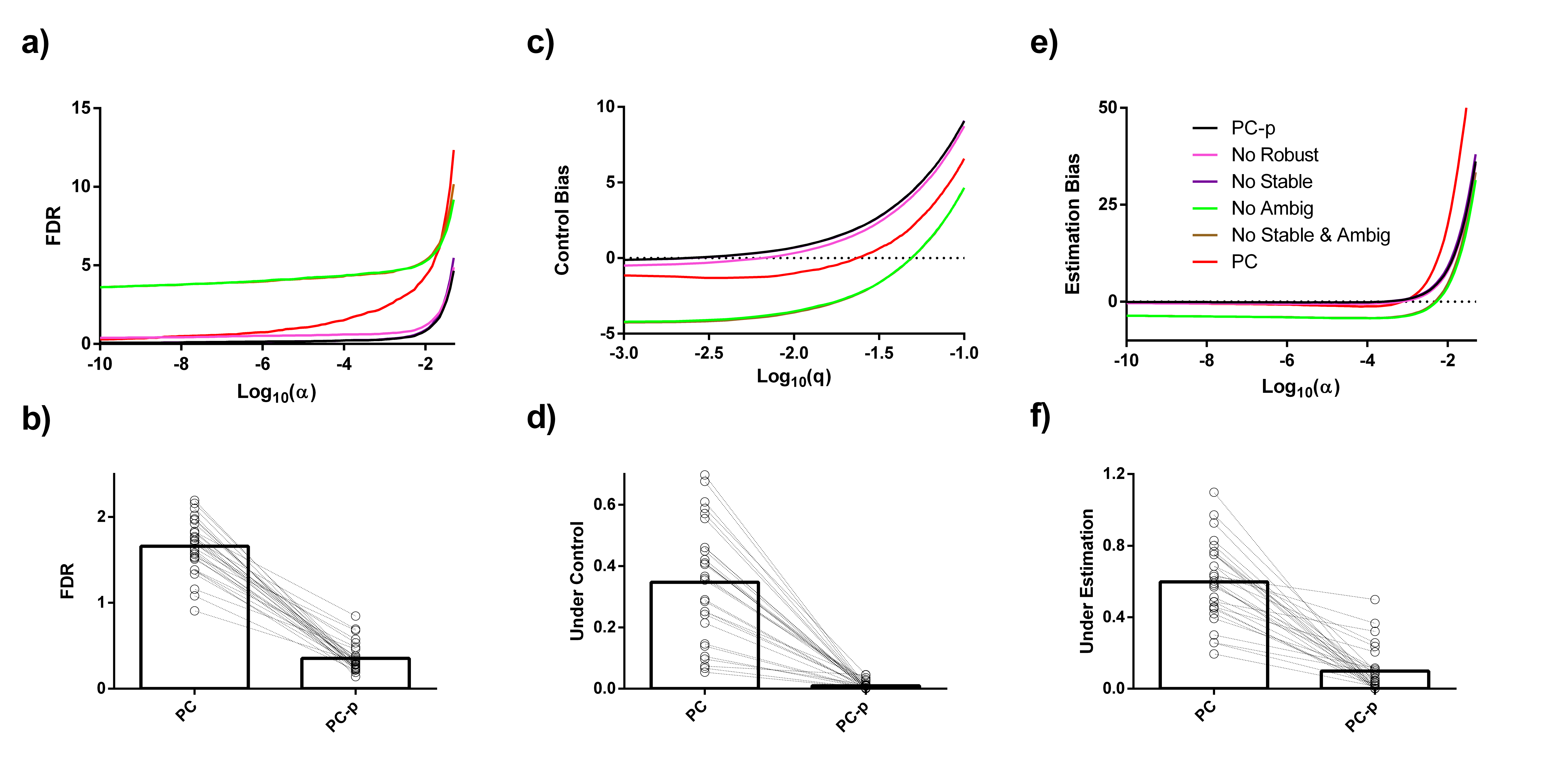}
\caption{Same setup as Figure \ref{fig_lowD} except with high dimensional data. PC-p significantly outperformed all other methods except PC-p without stabilization in terms of under-control and under-estimation.}
 \label{fig_highD}
\end{figure}
We next tested PC-p and the other five algorithms on high dimensional graph estimation. To do this, we generated thirty 100, thirty 200 and thirty 300 variable DAGs. We subsequently converted each of the DAGs to one linear non-recursive SEM-IE. Finally, we generated 1000 samples from each SEM-IE in order to obtain sample size to variable ratios of 10, 5 and 3.333.

Results are summarized using the FDR, control bias, and estimation bias metrics as averaged over the DAGs and their parameters in Figure \ref{fig_highD}. PC-p achieved similar results in low dimensions as it did for high dimensions. Specifically, PC-p obtained lower mean FDR values across the same $\alpha$ thresholds than PC-p without robust p-values (Figures \hyperref[fig_highD]{9a} and \hyperref[fig_highD]{9b}; $z=\text{-}4.703, p=2.563\text{E-}6$), PC-p without stabilization ($z=\text{-}2.232, p=0.026$), PC-p without ambiguation ($z=\text{-}4.782, p=1.734E-6$), PC-p without both stabilization and ambiguation ($z=\text{-}4.782, p=1.734\text{E-}6$), and PC original ($z=\text{-}4.782, p=1.734\text{E-}6$). Moreover, PC-p achieved significantly lower mean under-control than four of the five competing methods (Figures \hyperref[fig_highD]{9c} and \hyperref[fig_highD]{9d}; vs. no robust: $z=\text{-}4.623, p=3.790\text{E-}6$; vs. no ambig: $z=\text{-}4.782, p=1.734E-6$; vs. no stable \& no ambig: $z=\text{-}4.782, p=1.734\text{E-}6$; vs. PC original: $z=\text{-}4.782, p=1.734\text{E-}6$). PC-p did not outperform PC-p without stabilization at four of the five threshold $\alpha$ thresholds tested ($0.05: z=\text{-}1.121, p=0.262; 0.10: z=0.504, p=0.614$; $0.20: z=0.985, p=0.324$; $0.50: z=0.760, p=0.447$); however, PC-p did outperform PC-p without stabilization at an $\alpha$ threshold of 0.01 ($z=\text{-}3.692, p=2.225\text{E-}4$). Meanwhile, PC-p kept the mean over-control small at $4.507\%$ (SD: $0.240\%$). 

Results for estimation were again similar. PC-p achieved significantly lower mean under-estimation than the competing methods except PC-p without stabilization (Figures \hyperref[fig_highD]{9e} and \hyperref[fig_highD]{9f}; vs. all methods except no stable: $z=\text{-}4.782, p=1.734\text{E-}6$). PC-p did not outperform PC-p without stabilization at four of the five threshold $\alpha$ thresholds tested ($0.05: z=\text{-}1.820, p=0.069$; $0.10: z=0.175, p=0.861$; $0.20: z=0.625, p=0.532$; $0.50: z=0.608, p=0.543$); however, PC-p did outperform PC-p without stabilization at an $\alpha$ threshold of 0.01 ($z=\text{-}2.293, p=0.028$). PC-p also achieved a small degree of mean over-control ($2.129\%$, SD: $0.084\%$). 

We conclude that the results for control and estimation bias for high dimensional graph estimation are similar to the low dimensional case. However, stabilization only increased performance at lower $\alpha$ thresholds in the high dimensional scenario; we may nonetheless view this as a desirable property, since a lower $\alpha$ threshold helps the algorithm complete more quickly. 

\subsection{Real Data: CYTO}
\begin{figure}
\centering
    \includegraphics[width=1\textwidth]{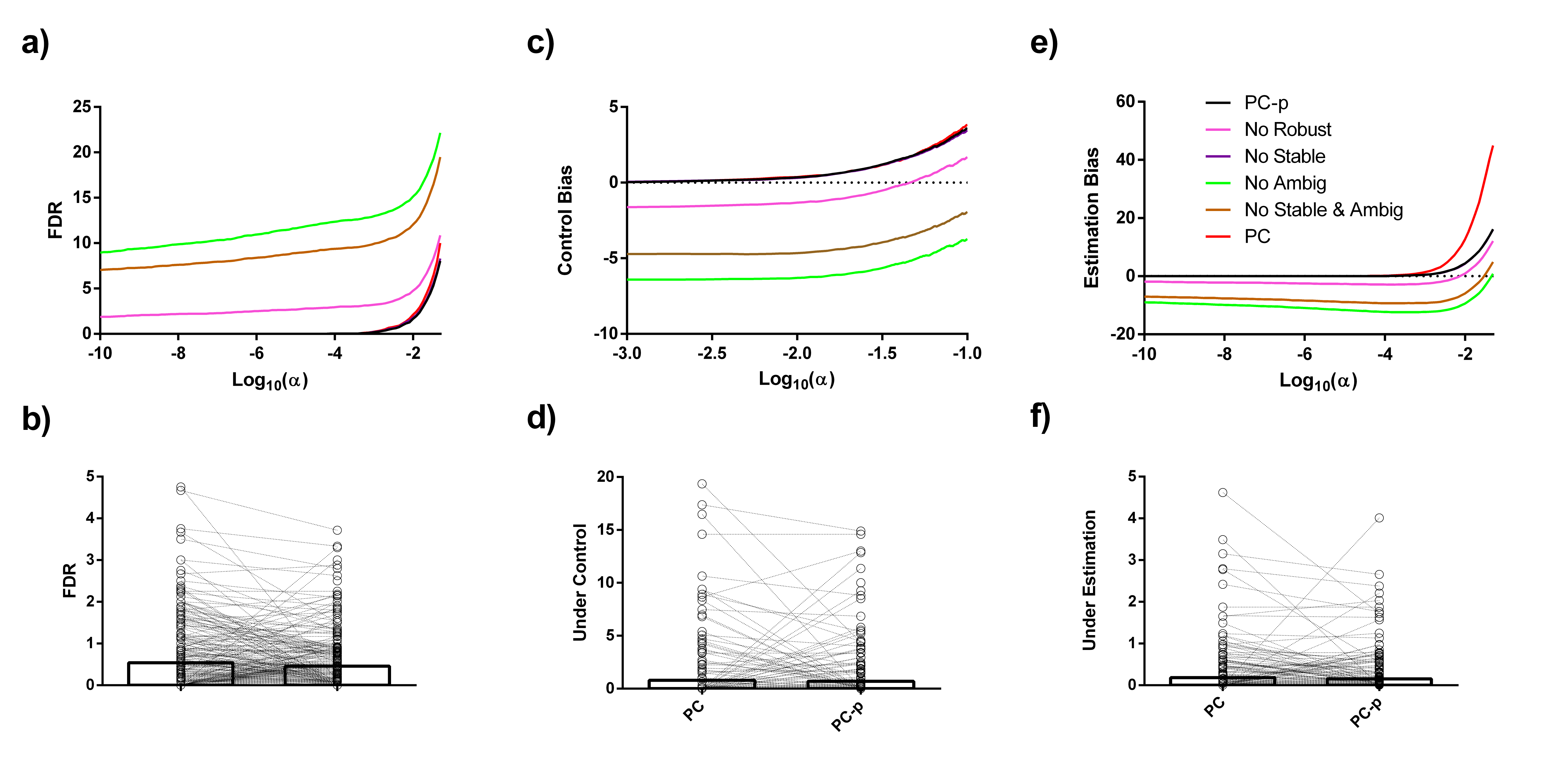}
\caption{Same setup as Figure \ref{fig_lowD} except with the CYTO dataset. PC-p significantly outperformed all methods across all metrics except the original PC algorithm in under-control.}
 \label{fig_real}
\end{figure}
We evaluated the six algorithms on the CYTO dataset which contains single cell recordings of the abundance of 11 phosphoproteins and phospholipids in human primary naive CD4+ T cells using flow cytometry \citep{Sachs05}. The variables in the dataset and their causal relationships can be represented as a DAG, where vertices are proteins or lipids and edges are phosphorylation interactions between the proteins and lipids. We used the general perturbation samples (i.e., CD3-CD28 and CD3-CD28-ICAM2) as our observational data; these perturbations are required to activate the phosphorylation pathways. Note that algorithms typically cannot accurately infer the gold standard solution set using the observational data alone, as noted by the original authors\footnote{In general, we do not have gold standard causal graphs for real data, so we must approximate the solution in some manner.}. As a result, we created a silver standard DAG by running LiNGAM as implemented in TETRAD V using default parameters on the full dataset of 1,755 samples; recall that LiNGAM is a method within a different class of causal discovery algorithms based on functional causal models. We then ran the six algorithms described in Section \ref{sec_6_1} on 1000 bootstrapped datasets of sample size 100 using Spearman's rho to handle the class of non-paranormal distributions.

We have summarized the results in Figure \ref{fig_real}. PC-p obtained lower mean FDR across the same $\alpha$ thresholds than PC-p without robust p-values ($z=\text{-}12.774, p=2.282\text{E-}37$), PC-p without stabilization ($z=-8.402, p=4.378\text{E-}17$), PC-p without ambiguation ($z=\text{-}24.598, p=1.343\text{E-}133$), PC-p without both stabilization and ambiguation ($z=\text{-}23.714, p=2.616\text{E-}124$), and original PC ($z=\text{-}4.924, p=8.469\text{E-}7$). Moreover, PC-p achieved significantly lower mean under-control than four of the five competing methods (vs. no robust: $z=\text{-}12.601, p=2.081\text{E-}36$; vs no stable: $z=\text{-}4.310, p=1.631\text{E-}5$; vs. no ambig: $z=\text{-}24.339, p=7.559\text{E-}131$; vs. no stable \& no ambig: $z=\text{-}22.893, p=5.515\text{E-}116$). PC-p did not outperform the original PC algorithm in mean under-control ($z=0.827, p=0.408$); however, PC-p did outperform the original PC algorithm in mean under-estimation $(z=-2.662, p=0.008)$. Meanwhile, PC-p kept the mean over-control small at $4.232\%$ (SD: $1.635\%$). PC-p also outperformed the other four methods in mean under-estimation (vs. no robust: $z=\text{-}12.684, p=7.289\text{E-}37$; vs. no stable: $z=\text{-}4.893, p=9.917\text{E-}7$; vs. no ambig: $z=\text{-}24.411, p=1.322\text{E-}131$; vs. no stable \& no ambig: $z=\text{-}23.301, p=4.333\text{E-}120$) while maintaining low mean over-estimation at $1.200\%$ (SD: $0.559\%$). We conclude that PC-p outperforms the other methods similar to the results with synthetic data. PC-p only outperformed PC in 2 of the 3 metrics, however, probably because the LiNGAM solution is only an estimate of the ground truth. 

\subsection{Real Data: GDP Dynamics}
\begin{figure}
\centering
    \includegraphics[width=1\textwidth]{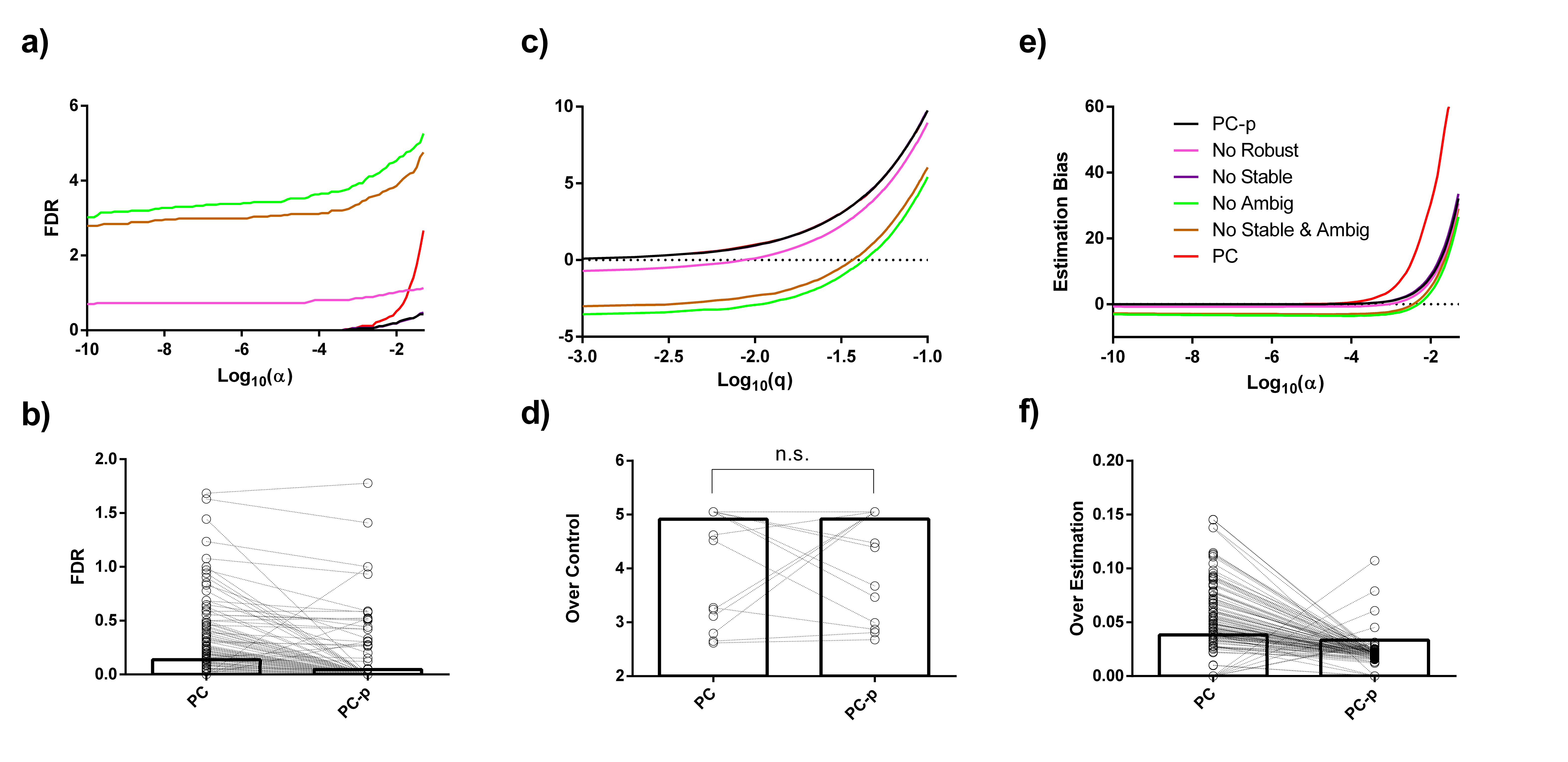}
\caption{Similar to Figure \ref{fig_lowD} except with the GDP dataset as well as over-control and over-estimation bias values instead of under. PC-p did not achieve lower under-control and under-estimation than PC, but it did achieve significantly lower mean FDR and over-estimation that PC.}
 \label{fig_real2}
\end{figure}
One way of approximating the underlying DAG involves learning the graph with a large number of samples. Another way uses time series data, where we know a priori that we must have contemporaneous causal relations or causal relations directed forward in time. In this experiment, we strip the time information from the six algorithms, and then identify the false discoveries when algorithms mistakenly detect a causal relation directed backwards in time. We used a time series dataset downloaded from the Economic Research Service of the United States Department of Agriculture containing ten economic indicators per year related to GDP among 192 countries\footnote{Web link: http://www.ers.usda.gov/data/macroeconomics/Data/HistoricalRealGDPValues.xls}. We specifically evaluated the algorithms on their ability to discover causal relations among the indicators within and between 1987, 1988 and 1989, where we treated each country as an i.i.d. sample and used 100 bootstrapped datasets.

We have summarized the results in Figure \ref{fig_real2}. PC-p again obtained lower mean FDR values across the $\alpha$ thresholds than PC-p without robust p-values ($z=\text{-}4.623, p=3.784\text{E-}6$), PC-p without ambiguation ($z=\text{-}8.054, p=4.128\text{E-}16$), PC-p without both stabilization and ambiguation ($z=\text{-}8.135, p=4.128\text{E-}16$), and original PC ($z=\text{-}6.624, p=3.500\text{E-}11$). However, PC-p did not obtain significantly lower mean FDR values than PC-p without stabilization ($\text{signed-rank}=70, p=0.600$). Next, PC-p achieved significantly lower mean under-control than three of the five competing methods including PC-p without robust p-values ($z=\text{-}4.374, p=1.218\text{E-}5$), PC-p without ambiguation ($z=\text{-}7.867, p=3.647\text{E-}15$), and PC-p without both stabilization and ambiguation ($z=\text{-}7.819, p=5.306\text{E-}15$). PC-p did not outperform PC-p without stabilization ($\text{signed-ranked}=3, p=1$) as well as the original PC algorithm ($\text{signed-ranked}=7, p=0.625$) in mean under-control; nevertheless, PC-p did outperform the former in over-control ($\text{signed-ranked}=3, p=0.020$) while keeping its own mean over-control low at $4.920\%$ (SD: $0.483\%$). PC-p also outperformed the same three methods in mean under-estimation (vs. no robust: $z=\text{-}4.372, p=1.229\text{E-}5$; vs. no ambig: $z=\text{-}7.818, p=5.363\text{E-}15$; vs. no stable \& no ambig: $z=\text{-}7.770, p=7.850\text{E-}15$) while maintaining low mean over-estimation at $2.032\%$ (SD: $0.281\%$). PC-p again did not outperform PC-p without stabilization ($\text{signed-rank}= 2, p=0.750$) and original PC ($\text{signed-rank}= 7, p=0.625$) in under-estimation, but it outperformed both in over-estimation (no stable: $z=\text{-}7.386, p=1.519\text{E-}13$; PC original: $z=\text{-}8.682, p=3.897\text{E-}18$). We conclude that PC-p outperforms most methods in either the under or over-metrics. The results however are not as clean as the results with the synthetic data because we only have access to a portion of the ground truth. 

\section{Conclusion} \label{sec_7}
We developed a new algorithm called PC-p which outputs a causal DAG with p-value bounds associated with each edge. One can then use the bounds with the BY procedure to achieve almost strong control and estimation of the FDR. The PC-p algorithm specifically integrates the skeleton discovery procedure of PC-stable, edge orientation with ambiguation, and robust hypothesis tests in order to accurately estimate p-values bounds while maintaining computational efficiency.

The PC-p algorithm represents the first global constraint-based method which can recover p-value estimates for every edge of a CPDAG. In our opinion, the algorithm is a significant advancement over previous methods which can only achieve strong control of the FDR under special conditions. Moreover, PC-p lays a foundation for developing similar methods which can also recover edge-specific p-values and achieve strong control of the FDR for graphs recovered by algorithms such as FCI and CCD. In particular, we suspect that a combination of the max and union bounds will also be sufficient for deriving upper bounds of the edge-specific p-values for more sophisticated constraint-based methods. The proposed approach may therefore represent one the earliest forms of a ``causal p-value.''

Now readers may wonder whether PC-p can also use the p-values to control the family-wise error rate (FWER). The answer is yes, and we recommend using the Benjamini-Holm step-down procedure as opposed to Hochberg's step-up procedure to control the FWER \citep{Hochberg88}, since the latter assumes positive dependency among the test statistics. However, application of an FWER controlling procedure to constraint-based causal discovery requires additional justification, since most investigators do not use constraint-based methods to definitively conclude causal relationships but rather to screen for potential causal variables. With the screening goal in mind, the FWER may be too conservative in practice, since it controls the rate of making a single Type I error across all of the hypothesis tests as opposed to controlling the proportion of Type I errors.

In summary, we introduced an algorithm called PC-p which outputs a causal DAG along with edge-specific p-value bounds. One can then use the BY procedure with the bounds to achieve almost strong control or estimation of the FDR and therefore assess the algorithm's confidence in each edge in a principled manner. We ultimately hope that this work will encourage more applications of constraint-based causal discovery to important problems in science.

\acks{Research reported in this publication was supported by grant U54HG008540 awarded by the National Human Genome Research Institute through funds provided by the trans-NIH Big Data to Knowledge initiative. The research was also supported by the National Library of Medicine of the National Institutes of Health under award numbers T15LM007059 and R01LM012095. The content is solely the responsibility of the authors and does not necessarily represent the official views of the National Institutes of Health.}


\newpage

\appendix
\section{Appendix}
\subsection{PC Algorithm Pseudocode} \label{appendix_1}
We provide pseudocode for the original PC algorithm. We summarize skeleton discovery in Algorithm \ref{pc_1}, unshielded v-structure discovery in Algorithm \ref{pc_2}, and orientation rule application in Algorithm \ref{pc_3}.

\begin{algorithm}[] \label{pc_1}
 \KwData{$\bm{X}^n$, $\alpha$}
 \KwResult{$\widehat{\mathbb{G}}$, $\mathcal{S}$}
 \BlankLine
 Form a completely connected undirected graph $\widehat{\mathbb{G}}$ on the vertex set $\bm{X}$\\
 $l=-1$ \\
 \Repeat{all ordered pairs of adjacent variables $(A,B)$ in $\widehat{\mathbb{G}}$ satisfy $|\widehat{\bm{N}}(A)\setminus B|\leq l$}{
 $l=l+1$ \\
 \Repeat{all ordered pairs of adjacent variables $(A,B)$ in $\widehat{\mathbb{G}}$ with $|\widehat{\bm{N}}(A)\setminus B| \geq l$ have been considered}{
 Select a new ordered pair of variables $(A, B)$ that are adjacent in $\widehat{\mathbb{G}}$ s.t. $|\bm{N}(A)\setminus B|\geq l$ \\
 \Repeat{$A - B$ is deleted from $\widehat{\mathbb{G}}$ or all $\bm{S} \subseteq \{\bm{N}(A)\setminus B \}$ with $|\bm{S}| =l$ have been considered}{Choose a new set $\bm{S} \subseteq \{\bm{N}(A)\setminus B\}$ with $|\bm{S}|=l$ using $order(\bm{X})$\\
 $p \leftarrow$ p-value from $test_{A \ci B | \bm{S}}$\\
 \If{$p > \alpha$}{Delete $A - B$ from $\widehat{\mathbb{G}}$ \\
 Insert $\bm{S}$ into $\mathcal{S}_{A,B}$ and $\mathcal{S}_{B,A}$ \\
 }
 }
}
}

 \caption{Skeleton Discovery}
\end{algorithm}

\begin{algorithm}[] \label{pc_2}
 \KwData{$\widehat{\mathbb{G}}$, $\mathcal{S}$}
 \KwResult{$\widehat{\mathbb{G}}$}
 \BlankLine
 \For{all ordered pairs of non-adjacent variables $(A, B)$ with common neighbor $C$}{
 \If{$C \not \in$ any set in $\mathcal{S}_{i,j}$}{
 Replace $A - C - B$ with $A \to C \leftarrow B$}
 }

 \caption{Unshielded V-structures} 
\end{algorithm}

\begin{algorithm}[] \label{pc_3}
 \KwData{$\widehat{\mathbb{G}}$}
 \KwResult{$\widehat{\mathbb{G}}$}
 \BlankLine
 \Repeat{there are no more edges to orient}{
 \uIf{$A - B$ and $\exists C$ s.t. $C \to A$, and $C$ and $B$ are non-adjacent}{
 Replace $A-B$ with $A \to B$
 }
 \uElseIf{$A - B$ and $\exists C$ s.t. $A \to C \to B$}{
 Replace $A - B$ with $A \to B$
 }
 \ElseIf{$A - B$ and $\exists B,D$ s.t. $A - C \to B$, $A - D \to B$, and $C$ and $D$ are non-adjacent}{
 Replace $A - B$ with $A \to B$
 } 
 
 }

 \caption{Orientation Rules}
\end{algorithm}

\subsection{Hypothesis Tests with Less Robust Bounds} \label{appendix_2}
We claimed to propose edge-specific hypothesis tests whose bounds are robust to Type II errors in Section \ref{sec_4_4}. We now explain our rationale.

Consider the following modification to \eqref{max_bound_gen}, where we have replaced the null with a series of logical conjunctions:
\begin{equation} \label{min_bound_gen1}
\begin{aligned}
H_0: &\bigwedge_{i=1}^m \text{oracle } i \text{ outputs } \neg P_i, \\
H_1: &\bigwedge_{i=1}^m \text{oracle } i \text{ outputs } P_i.
\end{aligned}
\end{equation}
We can bound the Type I error rate of the above hypothesis test as follows:
%\begin{equation} \label{non_robust_bound}
\begin{align*} \label{non_robust_bound}
\text{Pr(Type I error)}&=\text{Pr(}\bigwedge_{i=1}^m \text{test $i$ outputs $P_i$} |H_0 ) \\ \nonumber
&\leq\min_{i=1,\dots,m}{\text{Pr(test $i$ outputs $P_i$}  | H_0)} \\ \nonumber
&=\min_{i=1,\dots,m}{\text{Pr(test $i$ outputs $P_i$}  | \text{oracle $i$ outputs $\neg P_i$ )}} \\ \nonumber
&= \min_{i=1,\dots,m} g_i, \numberthis
\end{align*}
%\end{equation}

We can use the above bound with the following variant of \eqref{one_v_bound_f} for unshielded v-structures:
\begin{equation} \label{nonrobust0}
\begin{aligned}
H_0: &\text{ }\neg(A-C)\wedge\neg(B-C)\wedge \neg\gamma_{AB|C}, \\
H_1: &\text{ }(A-C)\wedge(B-C)\wedge \gamma_{AB|C},
\end{aligned}
\end{equation}
The above hypothesis test follows from the following natural hypothesis test:
\begin{equation} \label{nonrobust1}
\begin{aligned}
H_0: &\text{ All edges between $A,C,B$ are absent}, \\
H_1: &\text{ }(A-C)\wedge(B-C)\wedge \gamma_{AB|C},
\end{aligned}
\end{equation}
since the null of \eqref{nonrobust1} implies the null in \eqref{nonrobust0}.
From \eqref{non_robust_bound}, the Type I error rate of \eqref{nonrobust0} is bounded by $\min\{p_{A-C},p_{B-C},p_{\gamma_{AB|C} } \}$. This is a less robust bound than \eqref{one_v_bound_f} in terms of the Type II error rate, since failing to control one p-value can cause an algorithm to under-estimate the bound. For example, suppose the underlying truth corresponds to $p_{A-C}=0.01$, $p_{B-C}=0.03$, $p_{\gamma_{AB|C}}=0.02$. Thus, the Type I error rate of \eqref{nonrobust1} is truly bounded by $0.01$. However, suppose a Type II error causes PC-p to skip CI tests and therefore compute $p_{A-C}=0.01$, $p_{B-C}=0.03$, $p_{\gamma_{AB|C}}=0.003$, where the third term is under-estimated. Then, PC-p will under-estimate the bound at $0.003$ instead of the true $0.01$.

Note that generalizing \eqref{nonrobust1} to account for multiple possible ways of orienting a v-structure does not robustify the bound either, since we have:
\begin{equation}
\begin{aligned}
H_0: &\text{ All edges between } A,C,\bm{B}_1 \text{ are absent, and all edges between } \\ &\text{ } A,C,\bm{B}_2 \text{ are absent}, \\
H_1: &\text{ } (A-C)\wedge\Big(\Big[(\bm{B}_1-C)\wedge\gamma_{A\bm{B}_1 | C} ]\vee[(\bm{B}_2-C)\wedge\gamma_{A\bm{B}_2 | C} \Big]\Big).
\end{aligned}
\end{equation}
whose Type I error rate is bounded by $\min\Big\{p_{A-C},\min\{ p_{\bm{B}_1-C},p_{\gamma_{A\bm{B}_1 |C} } \}+\min\{p_{\bm{B}_2-C} ,p_{\gamma_{A\bm{B}_2 |C}} \} \Big\}$.

More broadly, we can consider the following hypothesis test:

\begin{equation} \label{min_bound_gen1}
\begin{aligned}
H_0: &\bigvee_{i=1}^m \Big( \bigwedge_{j=1}^{n_i} \text{oracle } i,j \text{ outputs } \neg P_{i,j} \Big), \\
H_1: &\bigwedge_{i=1}^m \Big( \bigwedge_{j=1}^{n_i} \text{oracle } i,j \text{ outputs } P_{i,j} \Big).
\end{aligned}
\end{equation}
We bound its Type I error rate as follows:
%\begin{equation} \label{under_bound2}
\begin{align*} \label{under_bound2}
&\text{Pr(Type I error)}=\text{Pr($\bigvee_{i=1}^m \bigwedge_{j=1}^{n_i}$ test $i,j$ outputs $P_{i,j} |H_0$ )} \\ \nonumber
&\leq\max_{i=1,\dots,m} \min_{j=1,\dots,n_i}{\text{Pr(test $i,j$ outputs $P_{i,j}$}  | H_0)} \\ \nonumber
&=\max_{i} \min_{j}{\text{Pr(test $i,j$ outputs $P_{i,j} |$ oracle $i,j$ outputs $\neg P_{i,j}$)}} \\ \nonumber
&= \max_{i} \min_{j} g_{i,j}, \numberthis
\end{align*}
%\end{equation}
The above bound is less robust to Type II errors than \eqref{max_bound_gen2}, since under-estimating one term in each group $i$ composed of $n_i$ terms can cause PC-p to also under-estimate \eqref{under_bound2}.

\subsection{Other Experimental Results} \label{appendix_3}

We have summarized the results for the low dimensional, high dimensional and real datasets across multiple $\alpha$ thresholds in Figures \ref{fig_appendix_lowD}, \ref{fig_appendix_highD}, \ref{fig_appendix_real} and \ref{fig_appendix_GDP} respectively. Relative differences in performance largely remained consistent across the thresholds, since PC-p usually achieved the lowest mean FDR, under-control and under-estimation values with minimal increases in mean over-control and over-estimation.

We have also summarized the results for adjacency discovery in Figures \ref{fig_appendix_lowD_skel}, \ref{fig_appendix_highD_skel}, and \ref{fig_appendix_real_skel}, where we tested whether the skeleton discovery procedure of PC with stabilization could improve the estimation of the p-value bounds relative to the procedure without stabilization. Results show that stabilization improves performance across the three metrics particularly with the low $\alpha$ threshold values of 0.01 and 0.05. Note that we cannot compute the same figures for the GDP dataset, since we can only evaluate relative performance levels based on edge direction in this case.

We have finally summarized the results using the structural Hamming distance in Figure \ref{fig_appendix_SHdis}. Notice that ambiguation helps PC-p achieve significantly lower Hamming distances across multiple $\alpha$ thresholds by forcing the algorithm to conservatively orient the edges. Again, we cannot compute the structural Hamming distances for the GDP dataset for the aforementioned reason.

\begin{figure}
\centering
    \includegraphics[width=1\textwidth]{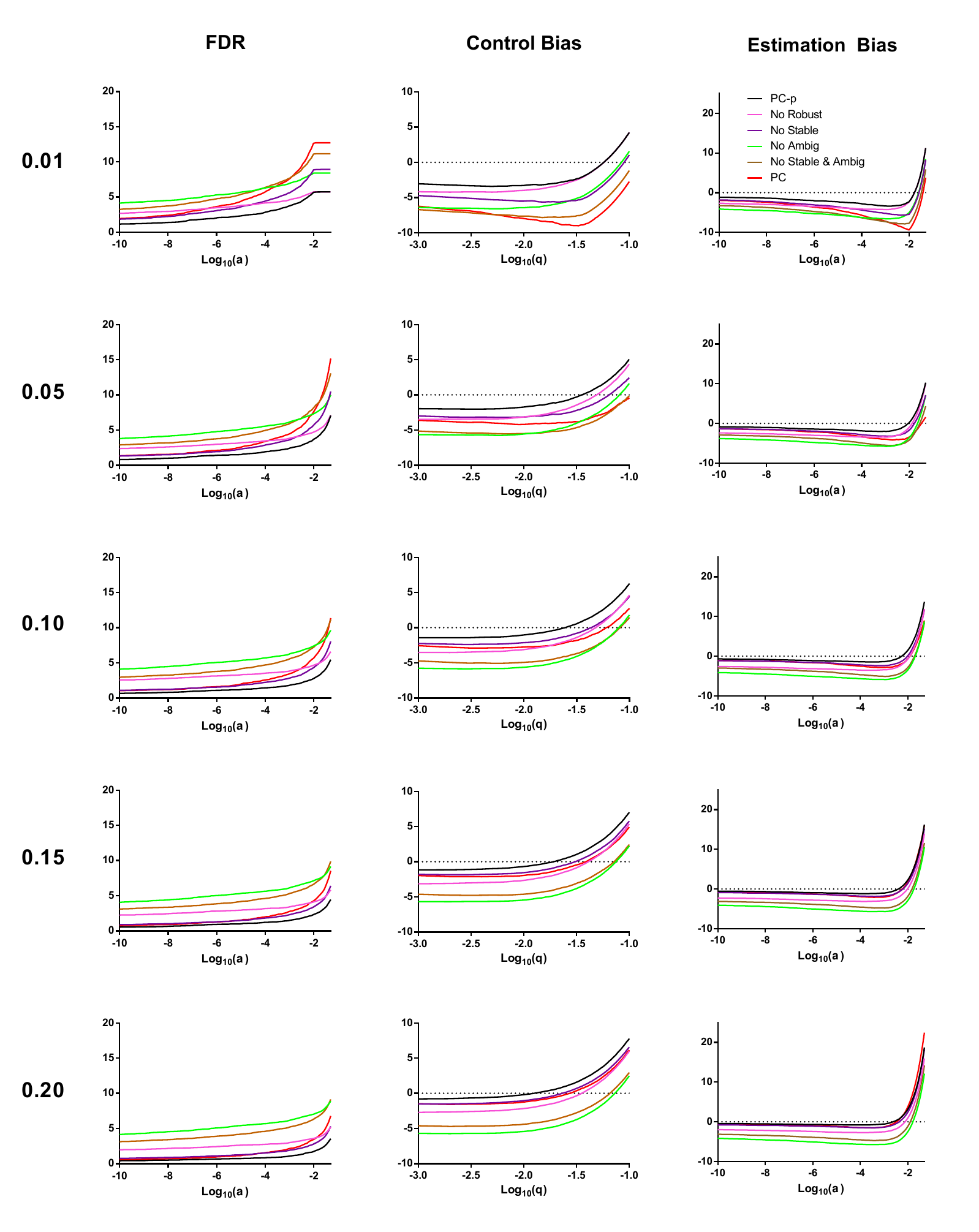}
\caption{Mean FDR, control bias, and estimation bias values across multiple $\alpha$ thresholds for the low dimensional datasets.}
 \label{fig_appendix_lowD}
\end{figure}

\begin{figure}
\centering
    \includegraphics[width=1\textwidth]{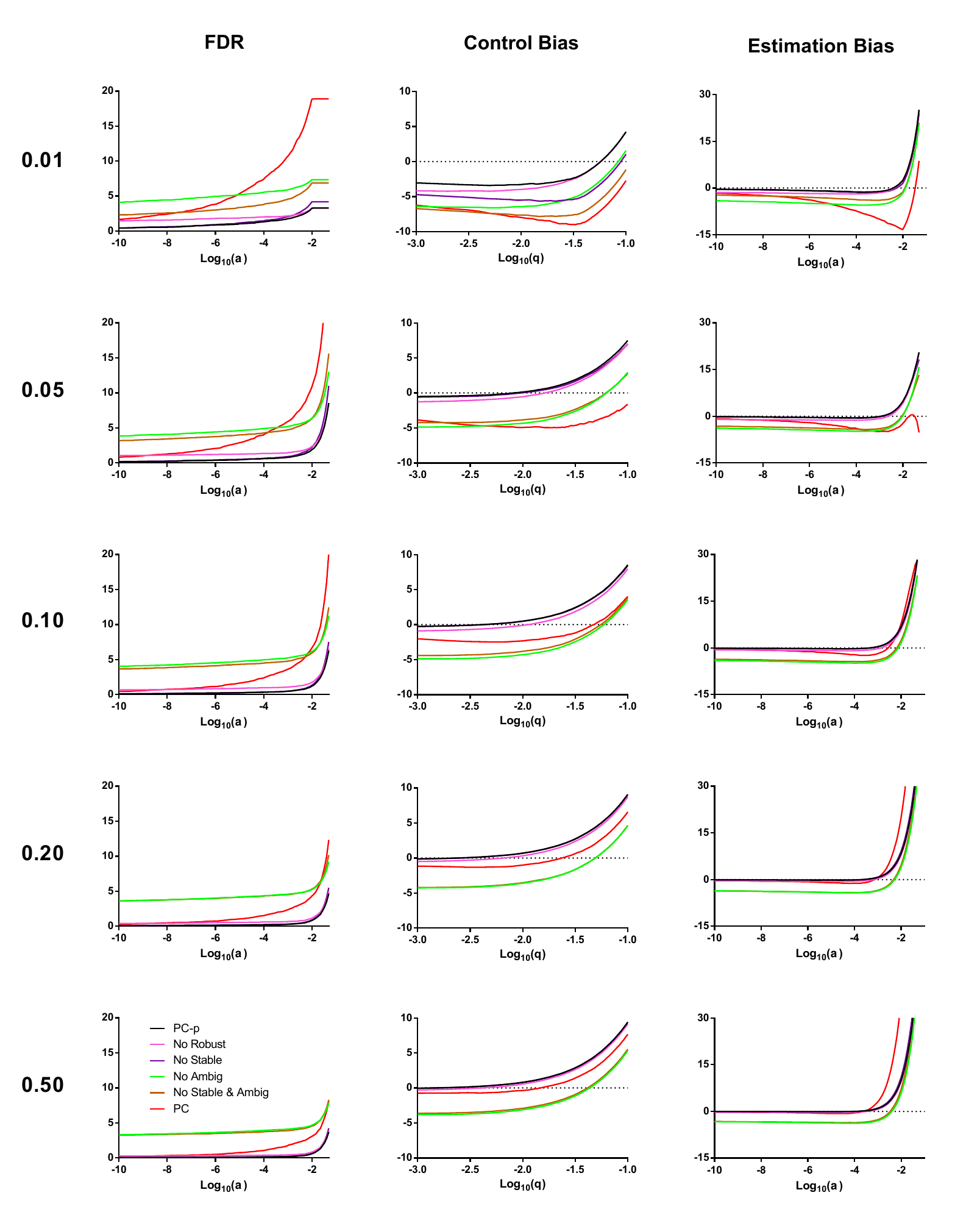}
\caption{Same as Figure \ref{fig_appendix_lowD} except with high dimensional datasets.}
 \label{fig_appendix_highD}
\end{figure}

\begin{figure}
\centering
    \includegraphics[width=1\textwidth]{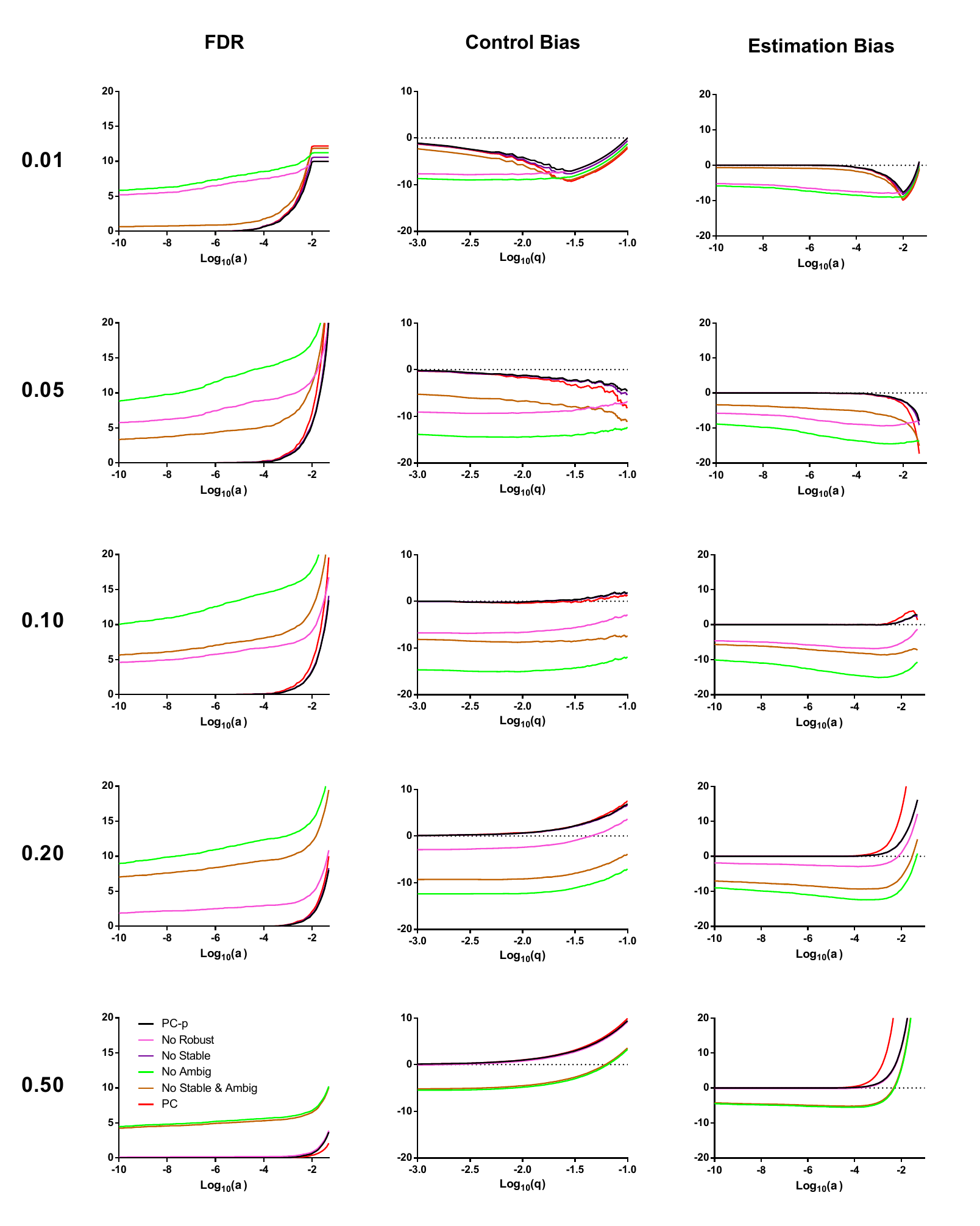}
\caption{Same as Figure \ref{fig_appendix_lowD} except with bootstrapped CYTO datasets.}
 \label{fig_appendix_real}
\end{figure}

\begin{figure}
\centering
    \includegraphics[width=1\textwidth]{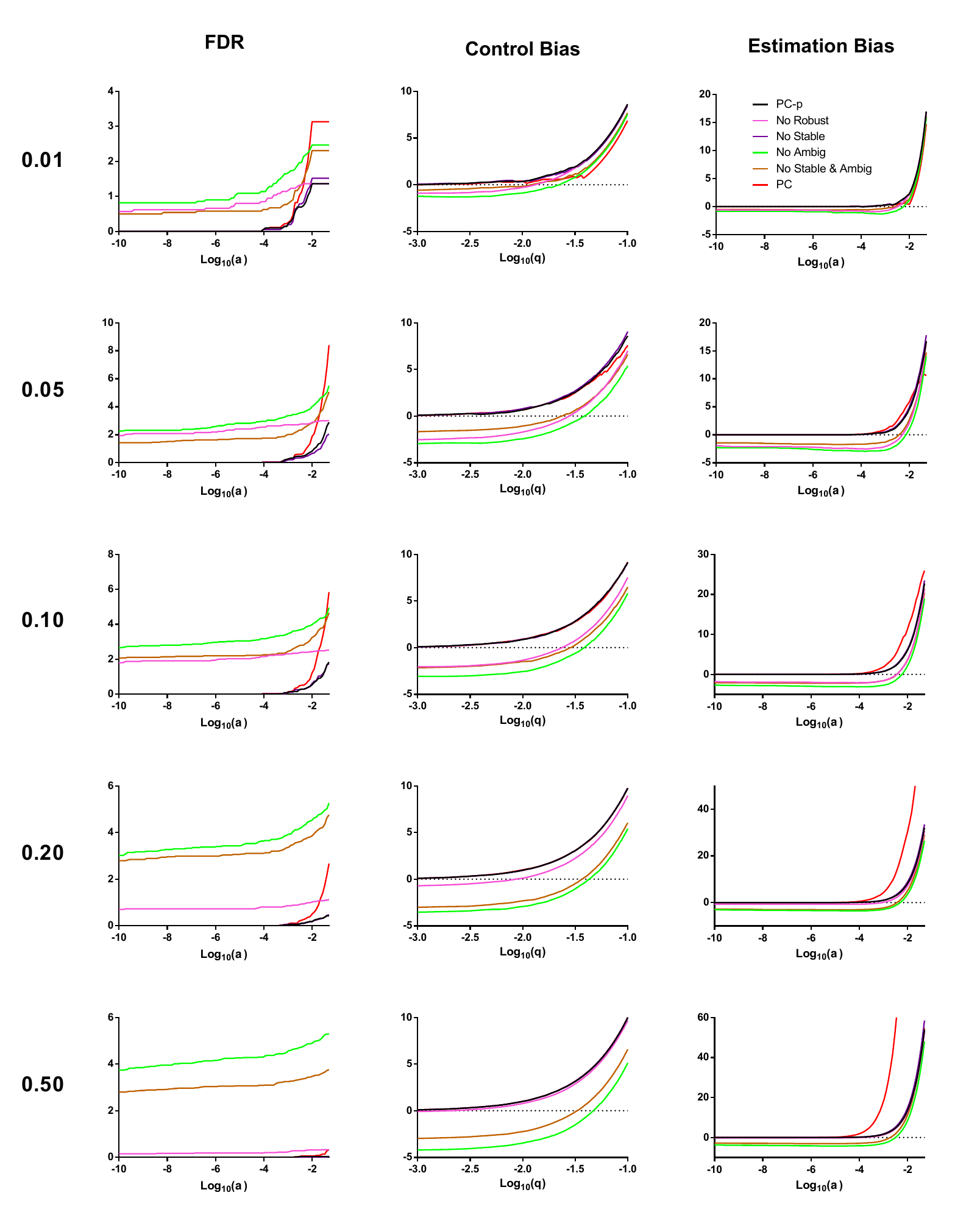}
\caption{Same as Figure \ref{fig_appendix_lowD} except with bootstrapped GDP datasets.}
 \label{fig_appendix_GDP}
\end{figure}

\begin{figure}
\centering
    \includegraphics[width=1\textwidth]{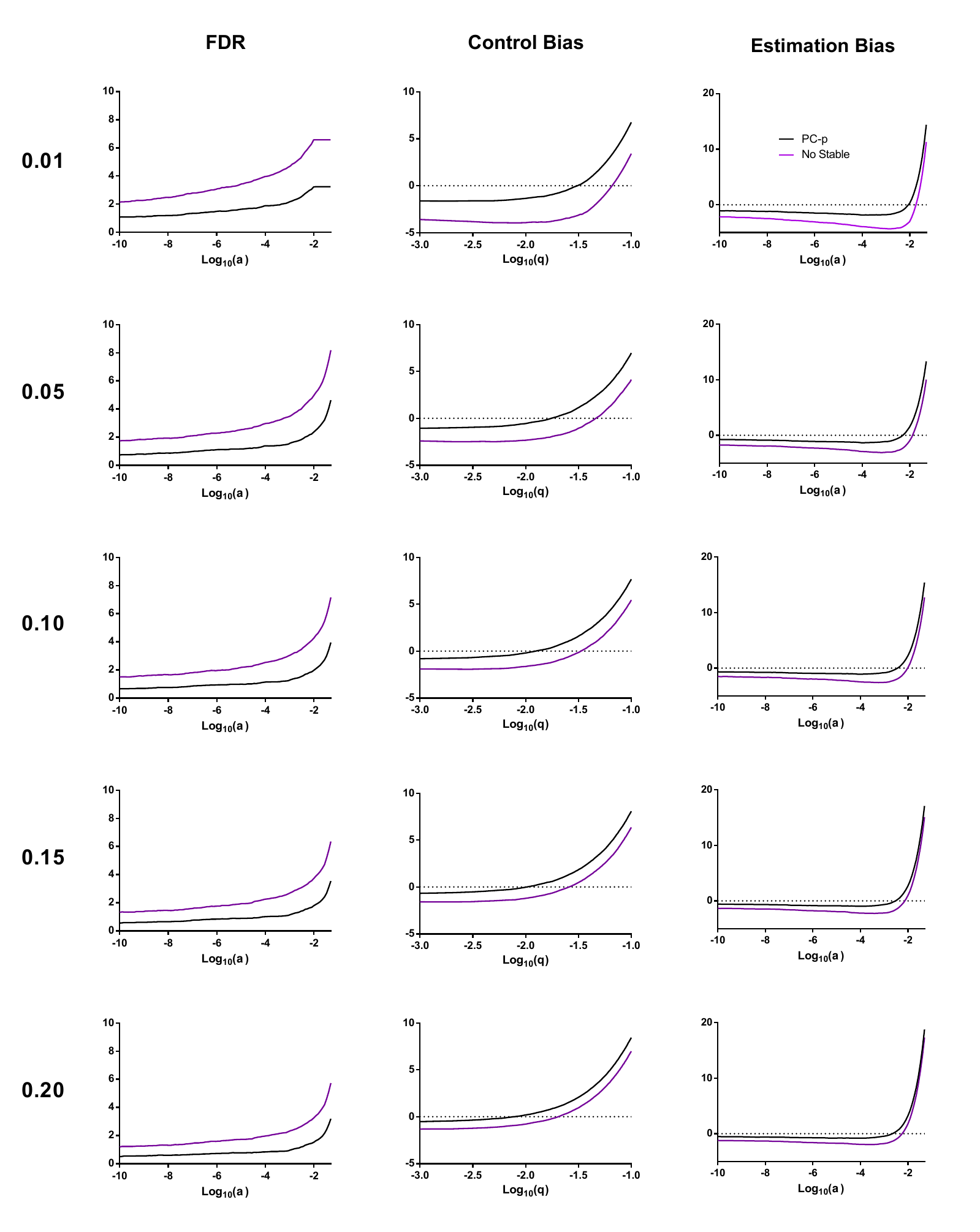}
\caption{Mean FDR, control bias, and estimation bias values across multiple $\alpha$ thresholds for the low dimensional datasets.}
\label{fig_appendix_lowD_skel}
\end{figure}

\begin{figure}
\centering
    \includegraphics[width=1\textwidth]{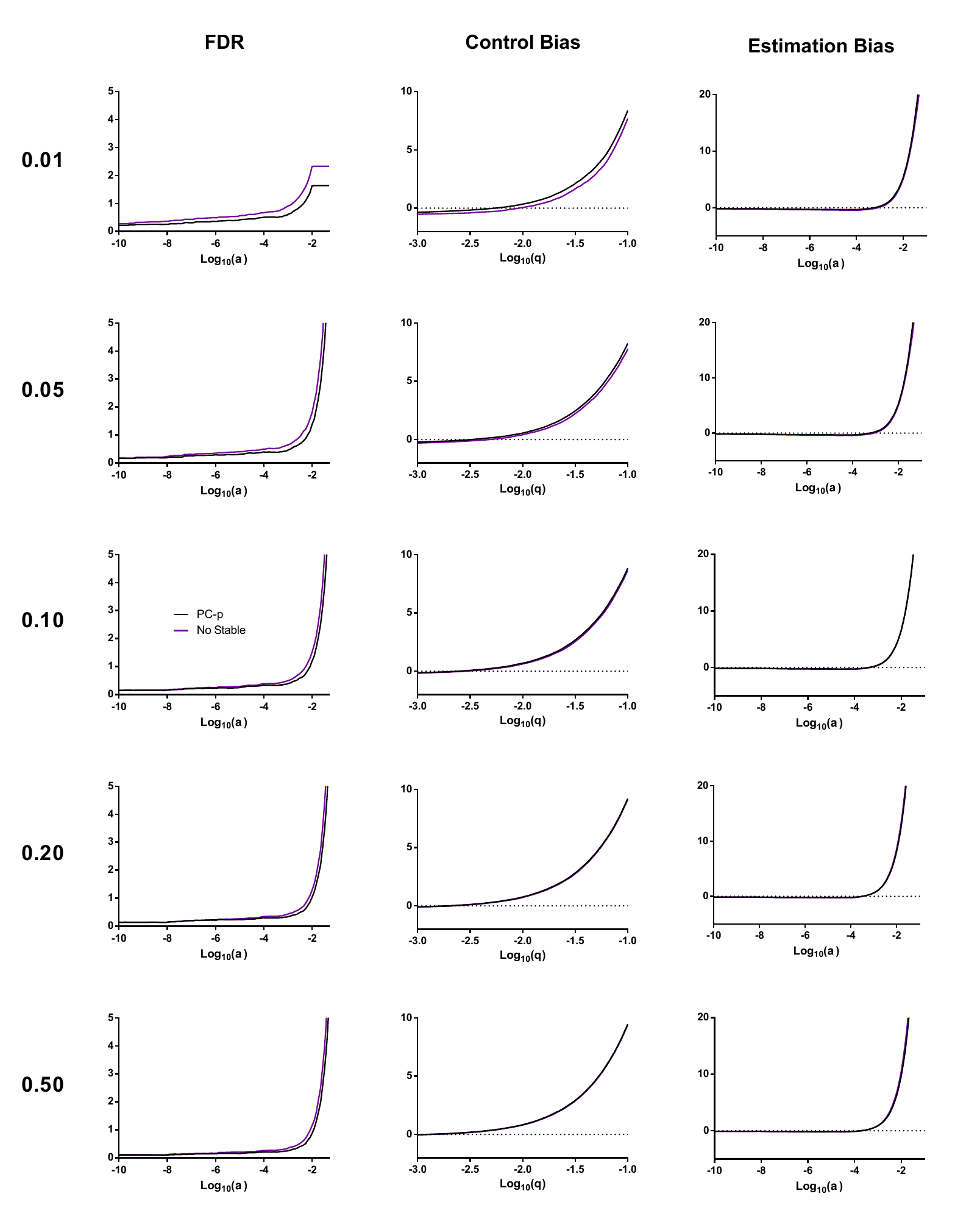}
\caption{Same as Figure \ref{fig_appendix_lowD_skel} except with high dimensional datasets.}
\label{fig_appendix_highD_skel}
\end{figure}

\begin{figure}
\centering
    \includegraphics[width=1\textwidth]{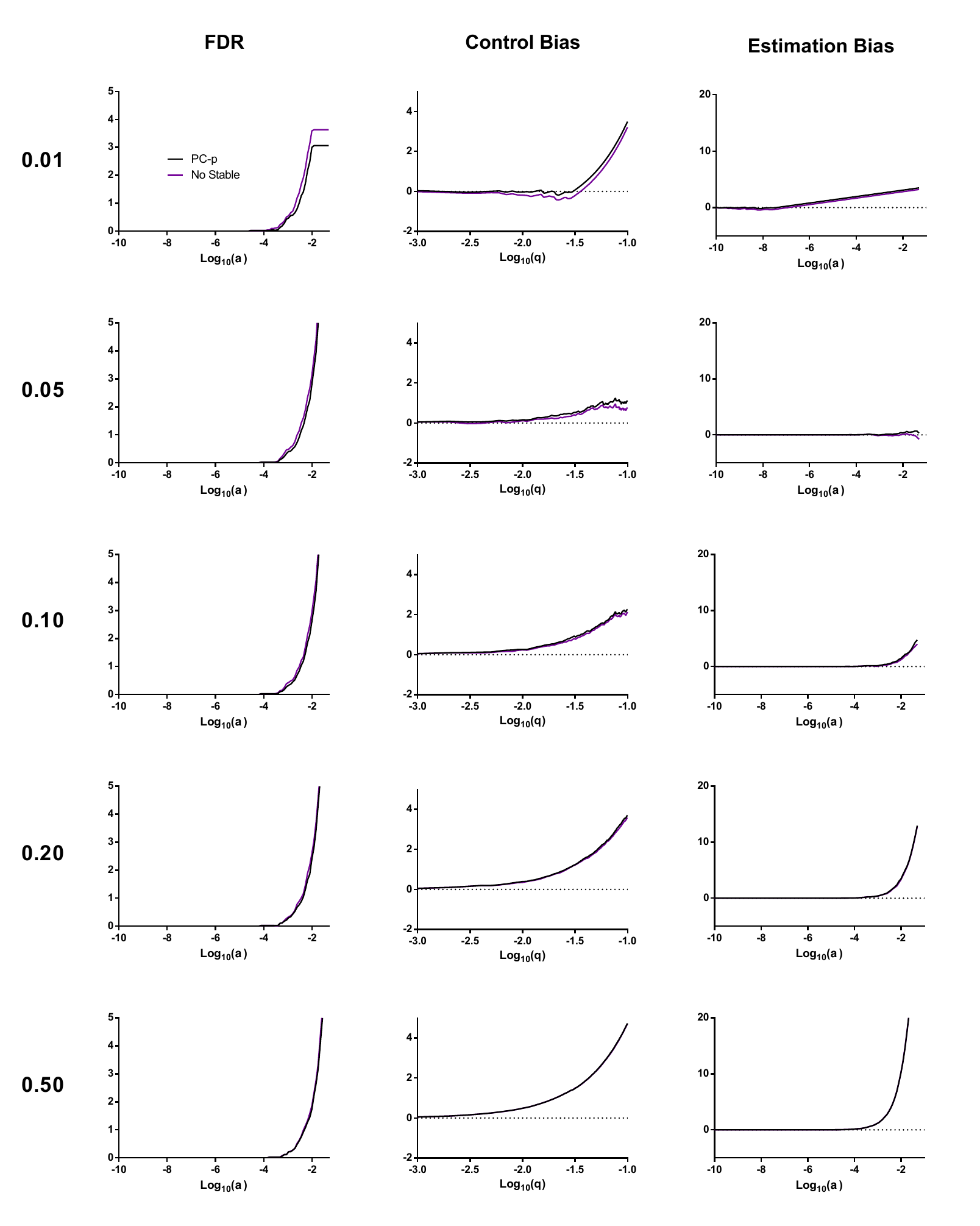}
\caption{Same as Figure \ref{fig_appendix_lowD_skel} except with bootstrapped CYTO datasets.}
 \label{fig_appendix_real_skel}
\end{figure}

\begin{figure}
\centering
    \includegraphics[width=1\textwidth]{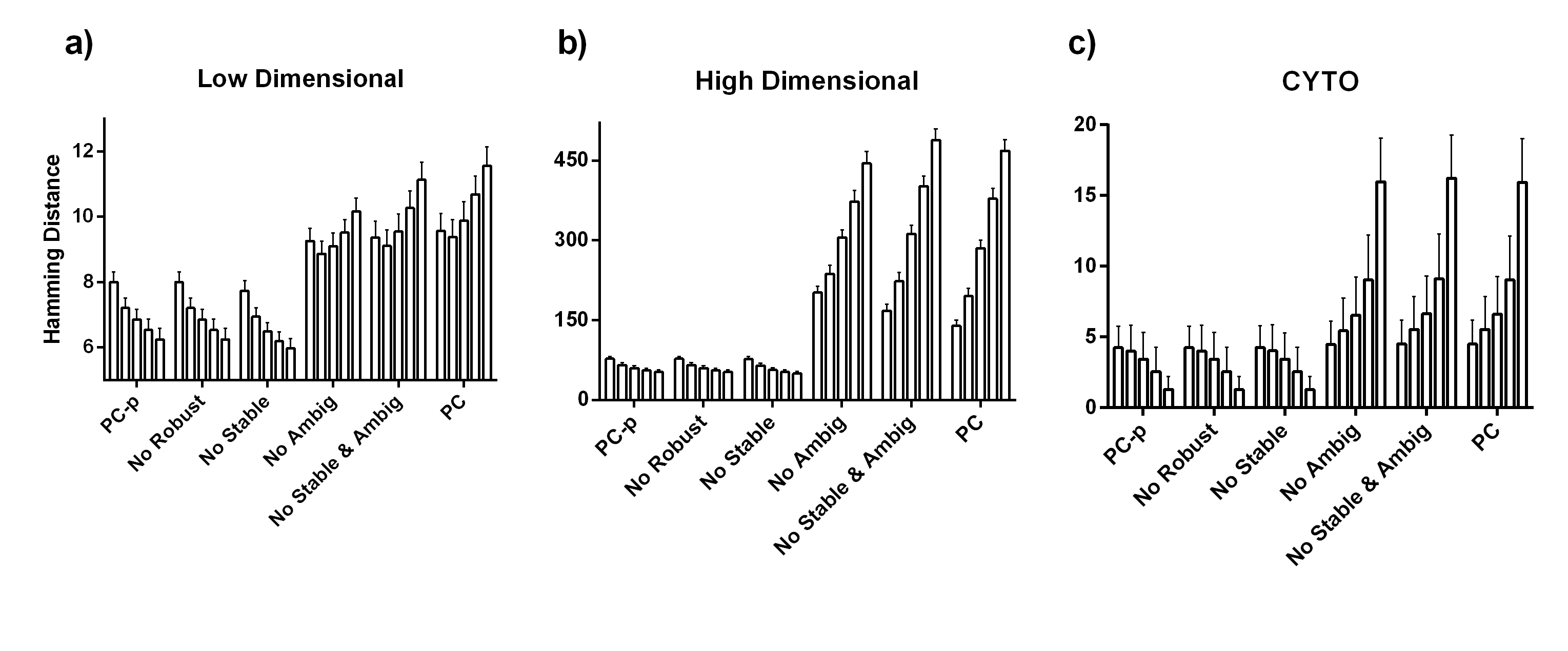}
\caption{Structural Hamming distances for the a) low dimensional, b) high dimensional and c) CYTO datasets. The three sub-figures associate 5 bars with each algorithm; these bars correspond to $\alpha$ thresholds of 0.01, 0.05, 0.1, 0.20 and 0.50 for the low dimensional and CYTO datasets, and $\alpha$ thresholds of 0.01, 0.05, 0.1, 0.15 and 0.20 for the high dimensional datasets. Error bars represent standard errors for a) and standard deviations otherwise.}
 \label{fig_appendix_SHdis}
\end{figure}

\vskip 0.2in
\bibliography{bibliography}

\end{document}